\documentclass[10pt,journal,compsoc]{IEEEtran}
\usepackage{amssymb,amsthm,amsmath}
\usepackage{booktabs}
\usepackage{bm,bbm}
\usepackage{cite}
\usepackage{color}
\usepackage{graphicx}
\usepackage{hyperref}
\usepackage{latexsym}
\usepackage{multirow}
\usepackage{newtxmath}
\usepackage{overpic}
\usepackage{url}
\newtheorem{definition}{Definition}
\newtheorem{assumption}{Assumption}
\newtheorem{theorem}{Theorem}

\newtheorem{proposition}{Proposition}
\newtheorem{lemma}{Lemma}

\newtheorem{claim}{Claim}
\newcommand{\bx}{{\bm{x}}}
\newcommand{\by}{{\bm{y}}}
\newcommand{\bfI}{{\mathbf{I}}}
\newcommand{\bfJ}{{\mathbf{J}}}
\newcommand{\bbI}{{\mathbbm{1}}}
\newcommand{\bbN}{{\mathbb{N}}}
\newcommand{\bbR}{{\mathbb{R}}}
\DeclareMathOperator{\Conv}{Conv}
\DeclareMathOperator{\cl}{cl}
\DeclareMathOperator{\argmax}{arg\,max}
\newcommand{\hyt}[1]{{\rm(\hypertarget{#1}{#1})}}
\newcommand{\hyl}[1]{{\rm(\hyperlink{#1}{#1})}}

\begin{document}
\title{Convergence Analysis of Mean Shift}
\author{Ryoya Yamasaki, and Toshiyuki Tanaka, \IEEEmembership{Member, IEEE}%
\IEEEcompsocitemizethanks{\IEEEcompsocthanksitem%
Ryoya Yamasaki and Toshiyuki Tanaka are with the Department of Informatics, 
Graduate School of Informatics, Kyoto University, Kyoto, 606-8501, Japan. 
E-mail: yamasaki@sys.i.kyoto-u.ac.jp, tt@i.kyoto-u.ac.jp.
This work has been submitted to the IEEE for possible publication. Copyright may be transferred without notice, after which this version may no longer be accessible.}}
\markboth{Preprint Version}
{Shell \MakeLowercase{\textit{et al.}}: Convergence Analysis of Mean Shift}
\IEEEtitleabstractindextext{%
\begin{abstract}
The mean shift (MS) algorithm seeks a mode of the kernel density estimate (KDE).
This study presents a convergence guarantee of the mode estimate sequence 
generated by the MS algorithm 
and an evaluation of the convergence rate, under fairly mild conditions, 
with the help of the argument concerning the {\L}ojasiewicz inequality.
Our findings extend existing ones covering analytic kernels and the Epanechnikov kernel.
Those are significant in that they cover the biweight kernel, 
which is optimal among non-negative kernels in terms of 
the asymptotic statistical efficiency for the KDE-based mode estimation.
\end{abstract}
}
\maketitle
\IEEEdisplaynontitleabstractindextext
\IEEEpeerreviewmaketitle
\IEEEraisesectionheading{%
\section{Introduction}
\label{sec:Intro}}
\IEEEPARstart{T}{he} mean shift (MS) algorithm \cite{fukunaga1975estimation, 
cheng1995mean, comaniciu2002mean} has been widely used 
in various fields such as computer vision, image processing, 
pattern recognition, and statistics.
One of its popular applications is data clustering
\cite{wu2007mean, chacon2018mixture}, 
where the MS algorithm is advantageous in that 
it does not need to specify the number of clusters in advance.
Other advantages of the MS-based clustering
compared with the $k$-means clustering are that
it does not require proper initialization of cluster centers,
as well as that it can cope with arbitrary cluster shapes.
Other applications of the MS algorithm include 
image segmentation \cite{comaniciu2002mean, tao2007color}, 
edge detection \cite{guo2005mean, zhu2009edge}, 
object tracking \cite{Coma2003OT, yang2005efficient}, 
and mode estimation \cite{parzen1962estimation, yamasaki2023optimal},
to mention a few.

The MS algorithm is an iterative algorithm that seeks a mode (local maximizer)
of the kernel density estimate (KDE).
Applications of the MS algorithm, such as data clustering and mode estimation, 
require the convergence of the mode estimate sequence generated by the MS algorithm.
It is therefore important to theoretically study 
convergence properties of the MS algorithm. 
However, as will be reviewed in Section~\ref{sec:Related},
available theoretical convergence guarantees of the MS algorithm 
which are applicable to practically relevant situations are quite limited:
As dynamical behaviors of the MS algorithm depend on the kernel
to be used in constructing the KDE,
convergence properties should also depend on the choice of the kernel. 
To the best of the authors' knowledge, the MS algorithm 
for multi-dimensional data has been shown to converge when
the Epanechnikov kernel \cite{comaniciu1999mean, Huang2018} 
or an analytic kernel \cite{yamasaki2019ms} is used.
These results do not cover practically relevant cases where
a piecewise polynomial kernel other than the Epanechnikov kernel is used.
Furthermore, little is known about the convergence rate of the MS algorithm.

In this paper we study convergence properties 
of the MS algorithm under some generic assumptions on the kernel. 
From a technical point of view, 
we follow a line similar to that of \cite{yamasaki2019ms} that focused on 
the {\L}ojasiewicz property \cite{lojasiewicz1965ensembles, absil2005convergence}:
this property ensures that a function under consideration is not too flat around its critical point,
and allows us to transfer a simpler convergence analysis of a sequence of KDE values 
for mode estimates into a convergence analysis of the mode estimate sequence itself.
More concretely, we make use of more advanced results 
\cite{kurdyka1994wf, attouch2013convergence, frankel2015splitting} about that property, 
to further extend the convergence analysis \cite{yamasaki2019ms} for analytic kernels to 
that for kernels characterized in terms of subanalyticity~\cite{kurdyka1994wf} in relation to the {\L}ojasiewicz property:
this extension allows us to obtain novel results, which include a convergence guarantee of 
the mode estimate sequence (Theorems~\ref{thm:MS-GCG} and \ref{thm:MS-CG}) 
and a worst-case bound of the convergence rate (Theorems~\ref{thm:MS-Rate} 
and \ref{thm:MS-Upper}) of the MS algorithm for a wider class of the kernels. 
%
Our contributions are of significance as the class of the kernels
we focus on in this study contains the biweight kernel,
which is known to be optimal among non-negative kernels 
in terms of the asymptotic statistical efficiency for 
the KDE-based estimation of a non-degenerate mode 
\cite{granovsky1991optimizing, yamasaki2023optimal}.

This paper is organized as follows.
We formulate the MS algorithm in Section~\ref{sec:Target},
and review related work on the convergence analysis 
of the MS algorithm in Section~\ref{sec:Related}.
In Section~\ref{sec:Prelim}, we describe the {\L}ojasiewicz property, 
and summarize the class of functions having that property.
On the basis of these preliminaries and abstract convergence theorems 
by \cite{attouch2013convergence, frankel2015splitting}, we provide 
a novel sufficient condition to ensure the convergence of the MS algorithm
and an evaluation of the convergence rate in Section~\ref{sec:Main}.
In Section~\ref{sec:Concl}, we conclude this paper,
and furthermore, we mention variants of the MS algorithm 
to which the analysis of this paper can be applied similarly, 
and possible directions for future research.
Supplementary material provides proofs of the theoretical results.

\section{MS Algorithm}
\label{sec:Target}
Various applications of the MS algorithm stem from
the characterization that the MS algorithm is an optimization algorithm
seeking a local maximizer of the KDE. 
Given $n$ data points $\bx_1,\ldots,\bx_n\in\bbR^d$, 
the KDE is constructed as
\begin{align}
\label{eq:KDE}%
	f(\bx)\coloneq
	\frac{1}{n h^d}\sum_{i=1}^n K\biggl(\frac{\bx-\bx_i}{h}\biggr),
\end{align}
where $K:\bbR^d\to\bbR$ and $h>0$ are called the 
kernel and the bandwidth parameter, respectively. 
Throughout this paper, for the kernel $K$
we adopt the following assumption, which is common in
studies of the MS algorithm: 
\begin{assumption}
\label{asm:RS}%
The kernel $K$ is bounded, continuous, non-negative, 
normalized, and radially symmetric.
\end{assumption}
The assumption of radial symmetry of the kernel $K$ 
leads to its alternative representation
\begin{align}
\label{eq:Pro-Ker}
	K(\bx)=\hat{K}(\|\bx\|^2/2)
\end{align} 
with what is called the profile $\hat{K}:[0,\infty)\to\bbR$ 
of $K$ and the Euclidean norm $\|\cdot\|$ in $\bbR^d$.

As mentioned by \cite{fashing2005mean, lange2016mm}, 
the MS algorithm can be seen as an example of 
the minorize-maximize (MM) algorithm under a certain condition.
The MM algorithm solves a hard original optimization problem
by iteratively performing construction of what is called a minorizer
of the original objective function and optimization of the minorizer. 
Let us write the right and left derivatives of $\hat{K}$,
if exist, as
\begin{align}
	\hat{K}'(u\pm)=\lim_{\nu\to u\pm0}\frac{\hat{K}(\nu)-\hat{K}(u)}{\nu-u}.
\end{align}
We make the following assumption for the profile $\hat{K}$ of the kernel $K$: 
\begin{assumption}
\label{asm:QM}%
The kernel $K$ has a convex and non-increasing profile $\hat{K}$ 
satisfying $\hat{K}'(0+)>-\infty$.
\end{assumption}
For a real-valued function $g$ defined on $S\subseteq\bbR$,
the subdifferential $\partial g(u)$ of $g$ at $u\in S$ is defined as
the set of values $c\in\bbR$ such that $g(v)-g(u)\ge c(v-u)$ holds 
for any $v\in S$. 
Under Assumption~\ref{asm:QM}, since the profile $\hat{K}$ is convex,
the subdifferential $\partial\hat{K}(u)$ is non-empty
for any $u\in(0,\infty)$ and given by $[\hat{K}'(u-),\hat{K}'(u+)]$.
Note that $\partial\hat{K}(0)=(-\infty,\hat{K}'(0+)]$ is non-empty as well 
under the assumption $\hat{K}'(0+)>-\infty$.
Since $\partial\hat{K}(u)$ is non-empty for any $u\in[0,\infty)$, 
one can show that the subdifferential $\partial\hat{K}(u)$
is non-decreasing in the sense that for $0\le u<v$ one has
$\max\partial\hat{K}(u)\le\min\partial\hat{K}(v)$:
Indeed, for any $u,v$ with $0\le u<v$,
take any $c_u\in\partial\hat{K}(u)$ and $c_v\in\partial\hat{K}(v)$.
From the definition of the subdifferential, one has 
$\hat{K}(v)-\hat{K}(u)\ge c_u(v-u)$ and $\hat{K}(u)-\hat{K}(v)\ge c_v(u-v)$,
which are summed up to $0\ge(c_u-c_v)(v-u)$, yielding $c_u\le c_v$.
See also \cite[Section 24]{rockafellar1997convex} for these properties
of subdifferentials of functions on $\bbR$.
Furthermore, as the profile $\hat{K}$ is non-increasing,
for any $u\in[0,\infty)$ one has $\max\partial\hat{K}(u)\le0$. 
Thus, defining a function $\check{K}$ on $[0,\infty)$ via 
\begin{align}
\label{eq:SD-Pro-Ker}%
	\check{K}(u)
	\begin{cases}
	\coloneq -\hat{K}'(0+)&\text{if }u=0,\\
	\in -\partial\hat{K}(u)&\text{if }u>0,
	\end{cases}
\end{align}
it is non-increasing, non-negative, and bounded since
$\check{K}(u)\le\check{K}(0)=-\hat{K}'(0+)<\infty$ for any $u\in[0,\infty)$
due to Assumption~\ref{asm:QM}.

As $-\check{K}(u)\in\partial\hat{K}(u)$, 
the definition of the subdifferential yields
$\hat{K}(v)-\hat{K}(u)\ge-\check{K}(u)(v-u)$ for any $u,v\in[0,\infty)$.
Substituting $(u,v)=(\|\bx'\|^2/2,\|\bx\|^2/2)$ into this inequality, one has 
\begin{align}
\label{eq:Mino-Ker}%
	K(\bx)\ge\bar{K}(\bx|\bx')\coloneq 
	K(\bx')+\frac{\check{K}(\|\bx'\|^2/2)}{2}(\|\bx'\|^2-\|\bx\|^2)
\end{align}
for any $\bx,\bx'\in\bbR^d$.
One also has $K(\bx')=\bar{K}(\bx'|\bx')$.
These properties imply that, under Assumptions~\ref{asm:RS} and \ref{asm:QM}, 
$\bar{K}(\bx|\bx')$ is a minorizer of the kernel $K$ at $\bx'$.

It should be noted that there is arbitrariness
in the definition~\eqref{eq:SD-Pro-Ker} of $\check{K}(u)$
at those values of $u$ at which $\partial\hat{K}(u)$ contains
more than a single value.
For example, the profile of the Epanechnikov kernel is given by
$\hat{K}(u)=C(1-u)_+$ with $C>0$, where $(\cdot)_+\coloneq\max\{\cdot,0\}$,
and thus $\partial\hat{K}(1)=[-C,0]$.
In this case one may adopt any value in the interval $[0,C]$ as
$\check{K}(1)$. 
Indeed, \cite{comaniciu1999mean} adopted $\check{K}(1)=C$,
whereas \cite{Huang2018} adopted $\check{K}(1)=0$.
We would like to note here that the following analysis is not affected
by how $\check{K}(u)$ is defined at such points.

The MS algorithm given a $t$th estimate $\by_t\in\bbR^d$ 
builds a minorizer of the KDE $f$ at $\by_t$ as
\begin{align}
\label{eq:Mino-KDE1}%
	\begin{split}
	\bar{f}(\bx|\by_t)
	&\coloneq\frac{1}{n h^d}\sum_{i=1}^n
	\bar{K}\biggl(\frac{\bx-\bx_i}{h}\biggl|\frac{\by_t-\bx_i}{h}\biggr)\\
	&=-\frac{1}{2 n h^{d+2}}\sum_{i=1}^n
	\check{K}\biggl(\biggl\|\frac{\by_t-\bx_i}{h}\biggr\|^2\biggr/2\biggr)\|\bx-\bx_i\|^2\\
	&+\text{($\bx$-independent constant)},
	\end{split}
\end{align}
which satisfies $\bar{f}(\by_t|\by_t)=f(\by_t)$ and 
$\bar{f}(\bx|\by_t)\le f(\bx)$ for any $\bx\in\bbR^d$.
Introduce a function
\begin{align}
\label{eq:coeff2}
	\check{f}(\bx)
	\coloneq\frac{1}{n h^d}\sum_{i=1}^n
	\check{K}\biggl(\biggl\|\frac{\bx-\bx_i}{h}\biggr\|^2\biggr/2\biggr),
\end{align}
with which the coefficient of the quadratic term $\|\bx\|^2$
in $\bar{f}(\bx|\by_t)$ is expressed as $-\check{f}(\by_t)/(2h^2)$. 
Assumption~\ref{asm:QM} ensures that $\check{f}(\bx)$ is non-negative
due to the non-negativity of $\check{K}(u)$. 
Furthermore, if $\check{f}(\by_t)=0$, then all the summands
on the right-hand side of \eqref{eq:coeff2} are zero
and hence the function $\bar{f}(\cdot|\by_t)$ is constant.
If $\check{f}(\by_t)>0$, on the other hand,
then the function $\bar{f}(\cdot|\by_t)$ is quadratic
and has a unique maximizer.

The MS algorithm then calculates the next estimate $\by_{t+1}$ as 
$\by_{t+1}\in\argmax_{\bx\in\bbR^d}\bar{f}(\bx|\by_t)$. 
More specifically, the MS algorithm calculates $\by_{t+1}$ via 
\begin{align}
\label{eq:MS-Iter}%
	\by_{t+1}
	=\by_t+\bm{m}(\by_t),
\end{align}
where
\begin{align}
	\bm{m}(\by)
	\coloneq\begin{cases}
	\bm{0}
	&\text{if }\check{f}(\by)=0,\\
	-\frac{\sum_{i=1}^n\check{K}(\|\frac{\by-\bx_i}{h}\|^2/2)(\by-\bx_i)}
	{\sum_{i=1}^n\check{K}(\|\frac{\by-\bx_i}{h}\|^2/2)}
	&\text{if }\check{f}(\by)\neq0,
	\end{cases}
\end{align}
with the all-zero vector $\bm{0}\in\bbR^d$.
The MS algorithm iterates the update rule \eqref{eq:MS-Iter} 
starting from a given initial estimate $\by_1\in\bbR^d$ while 
incrementing the subscript $t\in\bbN$.
Therefore, the MS algorithm can be regarded as an instance
of the MM algorithm.

Here, the update rule when $\check{f}(\by_t)=0$ is an exception-handling rule to
avoid the MS algorithm to be ill-defined due to the denominator ($=nh^d\check{f}(\by_t)$) 
of the ordinary update rule being zero. 
Under Assumptions~\ref{asm:RS} and \ref{asm:QM},
if $\check{f}(\by_t)=0$ then the gradient of the KDE $f$ also vanishes, 
that is, $\by_t$ is a critical point of $f$.
Therefore, the exception-handling rule ensures
the MS algorithm to stop at a critical point.
Also, the following proposition shows that no such exception 
occurs if one selects an initial estimate $\by_1$ properly:
\begin{proposition}
\label{prop:fw}
Assume Assumptions~\ref{asm:RS} and \ref{asm:QM}.
Let $(\by_t)_{t\in\bbN}$ be the mode estimate sequence
obtained by the MS algorithm~\eqref{eq:MS-Iter} starting from $\by_1$
with $f(\by_1)>0$.
Then, there exists a constant $C>0$ such that 
$\check{K}(\|\frac{\by_t-\bx_{i_t}}{h}\|^2/2)\ge C$ for some $t$-dependent 
index $i_t\in[n]\coloneq\{1,\ldots,n\}$, and consequently $\check{f}(\by_t)\ge\frac{C}{n h^d}$ for any $t\in\bbN$.
\end{proposition}
For example, in the data clustering with the MS algorithm \cite{wu2007mean, chacon2018mixture},
one adopts each data point $\bx_i$ as the initial estimate $\by_1$, 
and hence the additional assumption $f(\by_1)>0$ definitely holds.

The above construction of the MS algorithm as the MM algorithm shows
the ascent property $f(\by_t)=\bar{f}(\by_t|\by_t)\le\bar{f}(\by_{t+1}|\by_t)
\le f(\by_{t+1})$ of the density estimate sequence $(f(\by_t))_{t\in\bbN}$, 
and the boundedness of the KDE $f$ (due to Assumption~\ref{asm:RS}) 
guarantees the convergence of that sequence:
\begin{proposition}[{Theorem 1 in \cite{yamasaki2019ms}}]
\label{prop:dens}
Assume Assumptions~\ref{asm:RS} and \ref{asm:QM}.
For the mode estimate sequence $(\by_t)_{t\in\bbN}$ obtained by 
the MS algorithm \eqref{eq:MS-Iter} starting from any $\by_1\in\bbR^d$,
the density estimate sequence $(f(\by_t))_{t\in\bbN}$ is non-decreasing and converges.
\end{proposition}

The above proposition guarantees the convergence of
the density estimate sequence $(f(\by_t))_{t\in\bbN}$ generated by the MS algorithm.
From the application point of view, however, what we are interested in 
is not the convergence of the density estimate sequence $(f(\by_t))_{t\in\bbN}$
but that of the mode estimate sequence $(\by_t)_{t\in\bbN}$,
since it is the limit $\lim_{t\to\infty}\by_t$, if exists, that
will tell us the location of a mode or a cluster center. 
The difficulty here is that one cannot deduce the convergence of 
the mode estimate sequence $(\by_t)_{t\in\bbN}$ from the convergence of 
the density estimate sequence $(f(\by_t))_{t\in\bbN}$ without additional assumptions. 
Our main interest in this paper lies in convergence properties of 
the mode estimate sequence $(\by_t)_{t\in\bbN}$ obtained by the MS algorithm, 
such as whether it converges to a critical point, 
as well as its convergence rate when it converges.

\section{Related Work}
\label{sec:Related}
Convergence properties of the mode estimate sequence $(\by_t)_{t\in\bbN}$ 
have been discussed in several papers.
Some early convergence studies are, however, not rigorous. 
For instance, the proof in \cite{comaniciu2002mean} 
used an incorrect inequality evaluation to claim 
that the mode estimate sequence is a Cauchy sequence;
see counterexamples given in \cite{LiHuWu2007}.
Essentially the same flaw had been shared by \cite{arias2016estimation} 
in the discussion of consistency, which was subsequently amended 
in the errata \cite{Errata} to \cite{arias2016estimation}.
\cite{carreira2007gaussian} claimed the convergence of 
the mode estimate sequence under the assumption 
that the MS algorithm uses the Gaussian kernel,
on the basis of the fact that the MS algorithm under this assumption
is an example of the expectation-maximization (EM) algorithm
\cite{DempsterLairdRubin1977}. 
As pointed out by \cite{Aliyari2013}, 
however, this reasoning alone is not enough: 
the EM algorithm can be viewed as a sort of the MM algorithm and 
may not converge without additional conditions \cite{boyles1983convergence},
which is similar to the situation for the MS algorithm reviewed in Section~\ref{sec:Target}.

Later studies have successfully provided some sufficient conditions 
for the convergence of the mode estimate sequence.
In \cite{LiHuWu2007}, the convergence of the mode estimate sequence 
has been proved under the assumption that the KDE has a finite number 
of critical points inside the convex hull of data points.
For example, when the Epanechnikov kernel is used, 
the KDE is shown to have a finite number of critical points,
so that the result of \cite{LiHuWu2007} is applicable to
provide a convergence guarantee.
For the Epanechnikov kernel, something even stronger holds true: 
\cite{comaniciu1999mean} and \cite{Huang2018} proved 
that the MS algorithm converges in a finite number of iterations.
Another instance for which the finiteness of critical points,
and consequently the convergence of the mode estimate sequence, 
have been shown is the 1-dimensional KDE with the Gaussian kernel.
See, e.g., \cite{Silverman1981} and \cite{CarreiraPerpinanWilliams2003}. 
However, it is not known whether the number of critical points
of the KDE with the Gaussian kernel for the dimension $d\ge2$ is finite. 
See, e.g., \cite{Amendola2019}, where upper and lower bounds
of the number of \emph{non-degenerate} critical points were given,
whereas they wrote that
the finiteness of the number of critical points is still open. 
Although \cite{Aliyari2015} provided a condition under which 
the KDE with the Gaussian kernel has a finite number of critical points, 
his condition requires taking the bandwidth of the kernel large enough.
Under this condition, 
mode estimates to be obtained would have a large statistical bias.
Furthermore, the KDE with a large bandwidth might even yield a far smaller
number of mode estimates than the actual number of the true modes
when the data-generating distribution has multiple modes. 
Therefore, its practical significance is quite obscure, 
in view of applications of the MS algorithm 
such as data clustering and mode estimation.
Additionally, in the 1-dimensional case, \cite{Aliyari2013} proved 
the convergence of the mode estimate sequence for various kernels, 
by showing that its subsequence around a critical point becomes a bounded monotonic sequence.
However, this proof strategy cannot be extended to the multi-dimensional case.

More recently, \cite{yamasaki2019ms} have given a convergence proof of 
the MS algorithm using analytic kernels, including the Gaussian kernel.
Their proof takes advantage of the {\L}ojasiewicz property 
\cite{lojasiewicz1965ensembles, absil2005convergence} (see Definition~\ref{def:Loj})
of an analytic kernel and the corresponding KDE,
while not requiring assumptions either 
on the finiteness of critical points of the KDE, 
on the non-degeneracy of KDE's Hessian at critical points, 
on the size of the bandwidth,
or on the dimension of the data.
Thus, their result is significant in that it guarantees 
the convergence of the MS algorithm 
under practical settings on the bandwidth parameter,
and even in the multi-dimensional case.

To summarize,
it is only when the MS algorithm uses 
the Epanechnikov kernel \cite{comaniciu1999mean, Huang2018} 
or an analytic kernel \cite{yamasaki2019ms} that 
the convergence of the mode estimate sequence $(\by_t)_{t\in\bbN}$
has been guaranteed 
without regard to the size of the bandwidth parameter or 
the data dimension.

Much less is known so far about the convergence rate. 
Previous studies have clarified only 
the finite-time convergence when the algorithm uses 
the Epanechnikov kernel \cite{comaniciu1999mean, Huang2018}
and the linear convergence when the algorithm uses 
the Gaussian kernel and the KDE has a non-degenerate Hessian
at the convergent point \cite{carreira2007gaussian}.
The convergence rate when the Hessian is degenerate has not been clarified.

\section{Preliminaries: {\L}ojasiewicz Property}
\label{sec:Prelim}
As mentioned above,
\cite{yamasaki2019ms} proved the convergence of
the mode estimate sequence $(\by_t)_{t\in\bbN}$ of the MS algorithm
using an analytic kernel, without regard to the bandwidth parameter
or the data dimension.
The key in their proof is the {\L}ojasiewicz property/inequality for 
an analytic function \cite{lojasiewicz1965ensembles, absil2005convergence},
which provides a lower bound of the flatness of the function around its critical points.
In the convergence analysis of the MS algorithm,
this bound in turn allows us to transfer 
the convergence of the density estimate sequence $(f(\by_t))_{t\in\bbN}$ 
to that of the mode estimate sequence $(\by_t)_{t\in\bbN}$.
We follow a similar line to that of \cite{yamasaki2019ms}, 
but instead of relying on \cite{lojasiewicz1965ensembles, absil2005convergence}
as in~\cite{yamasaki2019ms}, in this paper 
we rely on \cite{kurdyka1994wf} that shows the {\L}ojasiewicz property for a wider class of functions beyond analytic ones, 
and on more advanced convergence analysis \cite{attouch2013convergence, frankel2015splitting}
that leverages that property.
%
We here describe the {\L}ojasiewicz property, 
and important classes of functions having that property.

We adopt the following definition of the {\L}ojasiewicz property/inequality,
along with related notions.
\begin{definition}[{\L}ojasiewicz property/inequality/exponent]
\label{def:Loj}
A function $g:S\to\bbR$ with $S\subseteq\bbR^d$
is said to have the \emph{{\L}ojasiewicz property} at $\bx'\in S$ with an exponent $\theta$, 
if there exists $\epsilon>0$ such that 
$g$ is differentiable on $U(\bx',g,S,\epsilon)\coloneq\{\bx\in S\mid \|\bx'-\bx\|<\epsilon, g(\bx')-g(\bx)\ge0\}$
and satisfies the \emph{{\L}ojasiewicz inequality}
\begin{align}
\label{eq:Lojasiewicz-ineq}
	\|\nabla g(\bx)\|
	\ge c \{g(\bx')-g(\bx)\}^\theta
\end{align}
with $c>0$, $\theta\in[0,1)$, and any $\bx\in U(\bx',g,S,\epsilon)$,
where we adopt the convention $0^0=0$ 
following \cite[Remark 4]{attouch2009convergence}.
Also, $g$ is said to have the \emph{{\L}ojasiewicz property} on $T\subseteq S$ (when $T=S$, we omit ``on T''),
if $g$ is differentiable on $T$ and there exists $\epsilon>0$ such that 
$g$ satisfies the {\L}ojasiewicz inequality \eqref{eq:Lojasiewicz-ineq} 
with $c>0$, $\theta\in[0,1)$, and any $(\bx',\bx)$ such that $\bx'\in T, \bx\in U(\bx',g,T,\epsilon)$. 
Moreover, the minimum value of $\theta$, 
with which $g$ has the {\L}ojasiewicz property at $\bx'$, 
is called the \emph{{\L}ojasiewicz exponent} of $g$ at $\bx'$.
\end{definition}

Intuitively, the {\L}ojasiewicz property of a function $g$ at $\bx'$
quantifies how flat the function $g$ is around the point $\bx'$.
It is obvious from the definition that, 
for any $\theta\in[0,1)$,
if $g$ has the {\L}ojasiewicz property at $\bx'$ with an exponent $\theta$,
then for any $\theta'\in[\theta,1)$ the same holds true with the exponent
$\theta'$ as well.
It is thus the minimum possible exponent $\theta$ (i.e., the {\L}ojasiewicz exponent) that is informative. 
If $g$ is continuously differentiable at $\bx'$
and if $\bx'$ is a non-critical point of $g$
(that is, $\nabla g(\bx')\not=\bm{0}$),
then for any $\theta\in[0,1)$,
$g$ trivially has the {\L}ojasiewicz property at $\bx'$
with the exponent $\theta$,
implying that $g$ is ``maximally non-flat'' at $\bx'$.
If, on the other hand, $\bx'$ is a local minimum of $g$,
then with a sufficiently small $\epsilon$ one has $U(\bx',g,S,\epsilon)=\{\bx'\}$,
implying that $g$ has the {\L}ojasiewicz property at the local minimum $\bx'$.
These facts demonstrate that Definition~\ref{def:Loj} is tailored 
primarily for characterizing the flatness of $g$ around its critical points
except local minima.

When $g$ is sufficiently smooth, its {\L}ojasiewicz exponent
at a critical point that is not a local minimum
is typically $\frac{1}{2}$, whereas it can be larger than that
if the Hessian of $g$ at the critical point is degenerate. 
%
As a more concrete example let us take 
\begin{align}
\label{eq:ExpA}
	g(\bx)=g(\bx')-\|\bx-\bx'\|^\alpha,
	\quad\alpha>1.
\end{align}
One then has
\begin{align}
	\|\nabla g(\bx)\|=\alpha\|\bx-\bx'\|^{\alpha-1}
	=\alpha\{g(\bx')-g(\bx)\}^{1-1/\alpha},
\end{align}
implying that $g$ has the {\L}ojasiewicz property at $\bx'$
with the exponent $\theta\ge1-1/\alpha$.
As one takes a larger $\alpha$, $g$ gets ``flatter'' at $\bx'$,
and correspondingly the {\L}ojasiewicz exponent
$1-1/\alpha$ becomes larger.
As another example, let
\begin{align}
	g(\bx)=g(\bx')-e^{-\|\bx-\bx'\|^{-\beta}}\bbI(\bx\not=\bx'),\quad\beta>0,
\end{align}
where $\bbI(c)$ is the indicator function that takes the value 1 
if the condition $c$ is true, and 0 otherwise. 
One then has
\begin{align}
	\begin{split}
	\|\nabla g(\bx)\|
	&=\beta e^{-\|\bx-\bx'\|^{-\beta}}
	\|\bx-\bx'\|^{-\beta-1}\bbI(\bx\not=\bx')\\
	&=\beta h(g(\bx')-g(\bx))
	\end{split}
\end{align}
with $h(z)=z(-\log z)^{1+1/\beta}\bbI(z>0)$, $z\ge0$,
on the basis of which one can show that $g$ does not have
the {\L}ojasiewicz property at $\bx'$ as defined in Definition~\ref{def:Loj}, 
that is, it is ``too flat'' at $\bx'$
to be captured by this definition%
\footnote{%
We would like to mention, however, that
an extended definition of the {\L}ojasiewicz property,
provided in supplementary material, allows us to capture the flatness
in this example as well.},
since for any $\theta\in[0,1)$ one has 
\begin{align}
	\frac{h(z)}{z^\theta}=z^{1-\theta}(-\log z)^{1+1/\beta}
	\stackrel{z\to+0}{\longrightarrow}0.
\end{align}

The significance of the {\L}ojasiewicz property for our purpose
is that it allows us to convert the convergence of the density estimate
sequence $(f(\by_t))_{t\in\bbN}$ into that of the mode estimate
sequence $(\by_t)_{t\in\bbN}$ when the KDE $f$ is ``not too flat,''
as well as that, when the mode estimate sequence $(\by_t)_{t\in\bbN}$
converges, the property can provide a guarantee of faster convergence
when $f$ is ``less flat'' at the limit, 
as will be discussed in Section~\ref{sec:Main}.

\cite{lojasiewicz1965ensembles} showed 
that analytic functions have the {\L}ojasiewicz property,
and thereafter, \cite{kurdyka1994wf} generalized that result 
to the class of $C^1$ functions with o-minimal structure
(see also \cite{van1996geometric}),
which particularly includes $C^1$ globally subanalytic functions:%
\footnote{%
\label{fn:Bolte}
More recently, \cite{bolte2007lojasiewicz, bolte2007clarke} 
extended the definition of the {\L}ojasiewicz inequality 
to the case of non-smooth functions, 
and showed that continuous globally subanalytic functions
satisfy that generalized {\L}ojasiewicz inequality.
Succeeding studies such as \cite{attouch2009convergence, attouch2013convergence, 
noll2014convergence, frankel2015splitting, bolte2016majorization}
used it to construct abstract convergence theorems for various optimization algorithms.
We also attempted convergence analysis according to such
a general framework that allows non-smooth objective functions,
but, even for the MS algorithm, we could not avoid the smoothness assumption
(assumption~\hyl{a1} in Theorem~\ref{thm:MS-GCG} or Assumption~\ref{asm:LCG} in Theorem~\ref{thm:MS-CG},
in Section~\ref{sec:SCP}).
Such difficulty is also discussed in \cite[Section 6]{noll2014convergence, frankel2015splitting}.
Therefore, from Section~\ref{sec:Prelim} onwards, 
we adopt a simple framework that supposes the smoothness 
even if it can be generalized to the non-smooth case.
%
Also, according to the boundedness assumption (Assumption~\ref{asm:RS}), 
we omit devices used to handle unbounded functions.}
\begin{proposition}[{\cite{lojasiewicz1965ensembles, kurdyka1994wf}}]
\label{prop:SA}
A function $g:S\to\bbR$ with $S\subseteq\bbR^d$ 
has the {\L}ojasiewicz property, 
if $g$ is analytic or if $g$ is $C^1$ globally subanalytic.
\end{proposition}

Now we introduce the definition of the global subanalyticity,
as well as several related notions,
the latter of which serve as sufficient conditions for the global subanalyticity. 
See also \cite{bolte2007lojasiewicz} and \cite{valette2022subanalytic}.
These notions are useful in practice,
because directly verifying the global subanalyticity can often be difficult,
whereas those sufficient conditions are easier to verify,
as in the discussion in supplementary material.
\begin{definition}[Global subanalyticity and related notions]\hfill
\label{def:GSF}
\begin{itemize}
\item%
A set $S\subseteq\bbR^d$ is called \emph{semialgebraic},
if there exist a finite number of polynomial functions $g_{ij}:\bbR^d\to\bbR$ such that 
$S=\bigcup_{i=1}^p\bigcap_{j=1}^q\{\bx\in\bbR^d\mid g_{ij}(\bx)\,\sigma_{ij}\,0\}$
with relational operators $\sigma_{ij}\in\{<,>,=\}$.
\item%
A set $S\subseteq\bbR^d$ is called \emph{semianalytic},
if for each point $\bx'\in\bbR^d$ there exist a neighborhood $T$ of $\bx'$ 
and a finite number of analytic functions $g_{ij}:T\to\bbR$ such that 
$S\cap T=\bigcup_{i=1}^p\bigcap_{j=1}^q\{\bx\in T\mid g_{ij}(\bx)\,\sigma_{ij}\,0\}$
with relational operators $\sigma_{ij}\in\{<,>,=\}$.
\item%
A set $S\subseteq\bbR^d$ is called \emph{subanalytic},
if for each point $\bx'\in\bbR^d$ there exist a neighborhood $T$ of $\bx'$ 
and a bounded semianalytic set $U\subseteq\bbR^{d+d'}$ with $d'\ge1$
such that $S\cap T=\{\bx\in\bbR^d\mid (\bx,\by)\in U\}$.
\item
A set $S\subseteq\bbR^d$ is called \emph{globally semianalytic} (resp.\,\emph{globally subanalytic}),
if its image under $\psi(\bx)=(x_1/(1+x_1^2)^{1/2},\ldots,x_d/(1+x_d^2)^{1/2})$ 
is a semianalytic (resp.\,subanalytic) subset of $\bbR^d$.
\item%
A function $g:S\to\bbR$ with $S\subseteq\bbR^d$ is called \emph{semialgebraic}
(resp.\,\emph{semianalytic}, \emph{subanalytic}, \emph{globally semianalytic}, or \emph{globally subanalytic}),
if its graph $\{(\bx,y)\in S\times\bbR\mid y=g(\bx)\}$ is semialgebraic
(resp.\,semianalytic, subanalytic, globally semianalytic, or globally subanalytic) 
subset of $\bbR^{d+1}$.
\item%
A function $g:S\to\bbR$ with $S\subseteq\bbR^d$ is called
\emph{piecewise polynomial} with the \emph{maximum degree} $k\in\bbN$, 
if there exists a finite collection $\{S_l\}_{l\in[L]}$ of subdomains $S_l\subseteq S$, $l\in[L]$, 
that forms a partition of $S$ (i.e., $S_l\neq\emptyset$ for all $l\in[L]$, 
$S_l\cap S_{l'}=\emptyset$ for all $l,l'\in[L]$ with $l\not=l'$, and $\cup_{l\in [L]}S_l=S$),
such that $g(\bx)=g_l(\bx)$ for any $\bx\in S_l$ 
(i.e., the restriction of $g$ to $S_l$ is the same as that of $g_l$ to $S_l$) 
with a polynomial $g_l:S\to\bbR$ for each $l\in[L]$
and that the maximum degree of $\{g_l\}_{l\in[L]}$ is $k$. 
\end{itemize}
\end{definition}

\begin{figure*}[!t]
\centering
\includegraphics[height=6cm, bb=0 0 602 256]{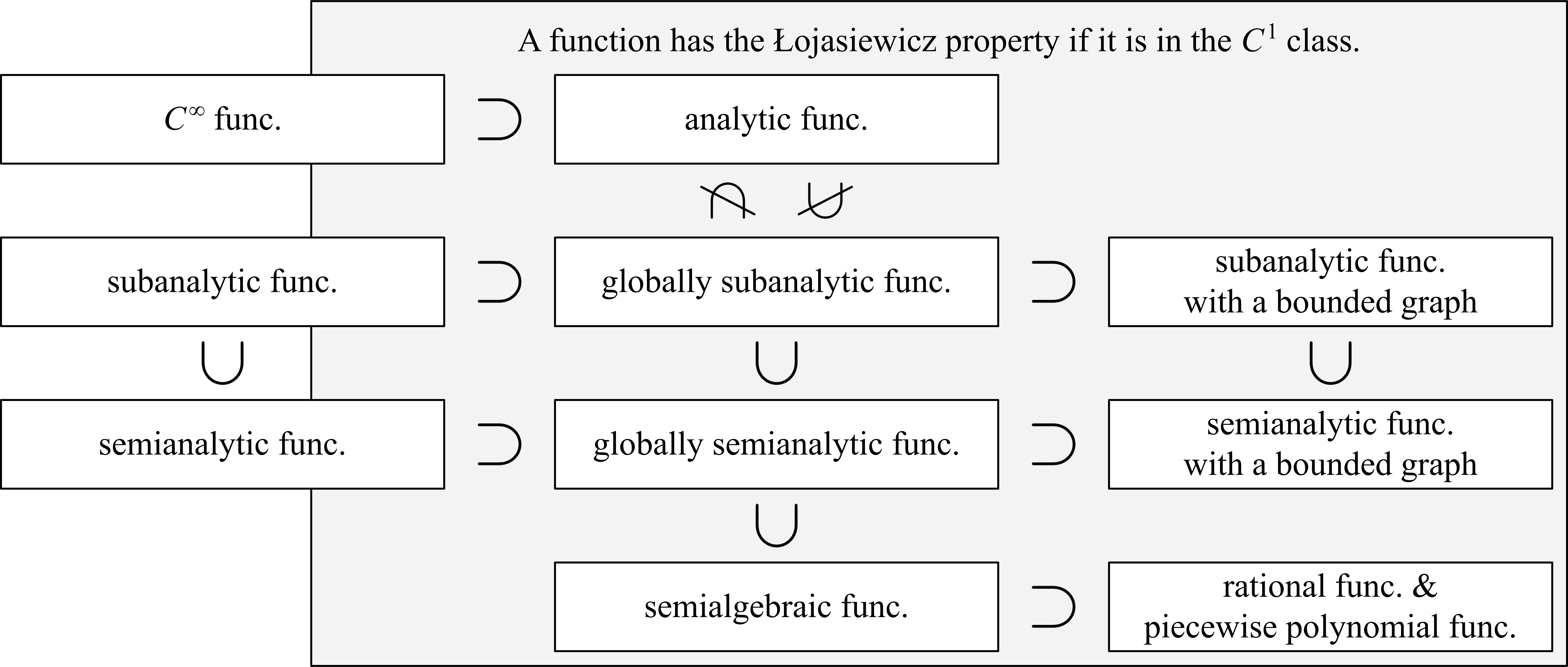}
\caption{%
Inclusion relation among important function classes 
relevant to the discussion on the {\L}ojasiewicz property.}
\label{fig:relation}
\end{figure*}

The class of semialgebraic functions has a wide variety of instances:
polynomial, rational, and more generally piecewise polynomial functions 
are semialgebraic \cite{bierstone1988semianalytic, dedieu1992penalty}.
As will be discussed in the next section, 
the class of piecewise polynomial functions that includes the biweight kernel 
is of particular importance in the discussion of this study.
Any globally semianalytic functions are semianalytic, 
and any semianalytic functions with a bounded graph are globally semianalytic 
\cite[before Example 1.1.4]{valette2022subanalytic}.
Any globally subanalytic functions are subanalytic, 
and any subanalytic functions with a bounded graph are globally subanalytic 
\cite[after Definition 2.2]{bolte2007lojasiewicz}.
Also, semianalytic functions are subanalytic (which can be seen from Definition~\ref{def:GSF}),
globally semianalytic functions are globally subanalytic \cite[Definition 1.1.6]{valette2022subanalytic},
and semialgebraic functions are globally semianalytic \cite[Example 1.1.4]{valette2022subanalytic}.
Note that an analytic function is not necessarily globally subanalytic
(of course, the converse is not necessarily true either:
a globally subanalytic function is not necessarily analytic).
For example, $g(x)=\sin(x)$, $x\in\bbR$, is certainly analytic but not globally subanalytic \cite[Example 1.1.7]{valette2022subanalytic}.
Moreover, it should be noted that a semianalytic/subanalytic function 
(e.g., the sine function defined on $\bbR$) and a $C^\infty$ function are not 
necessarily globally subanalytic and do not always have the {\L}ojasiewicz property;
the ``Mexican hat'' function (equation~(2.8) in~\cite{absil2005convergence})
and the function shown on page 14 of~\cite{palis2012geometric}
are instances that are of class $C^\infty$ and not globally subanalytic,
and these functions do not have the {\L}ojasiewicz property. 
These inclusion relations are summarized in Figure~\ref{fig:relation}.

As stated in Proposition~\ref{prop:SA},
in view of the {\L}ojasiewicz property, what is important
for our purpose is to provide sufficient conditions
for the KDE to be $C^1$ globally subanalytic.
Thus, sufficient conditions for global subanalyticity
in the above inclusion relations, as well as 
the stability of the global subanalyticity 
under the summation \cite[Properties 1.1.8]{valette2022subanalytic}, 
are important, which are summarized as follows:
\begin{proposition}
\label{prop:GSprop}
Any semialgebraic or globally semianalytic functions, 
any semianalytic or subanalytic functions with a bounded graph,
and the summation of any globally subanalytic functions
are globally subanalytic.
\end{proposition}

\section{Main Results: Convergence Theorems for MS Algorithm}
\label{sec:Main}
\subsection{Convergence to a Critical Point}
\label{sec:SCP}
In this subsection, 
we provide a sufficient condition for 
the mode estimate sequence $(\by_t)_{t\in\bbN}$ of the MS algorithm
to converge to a critical point of the KDE $f$.
Our result is along the same line as 
the existing convergence theorem by \cite{yamasaki2019ms} 
for the MS algorithm using analytic kernels,
and further extends it on the ground of Propositions~\ref{prop:SA} and \ref{prop:GSprop}
stating that $C^1$ globally subanalytic kernels and the corresponding KDE have the {\L}ojasiewicz property.

Several recent studies in optimization theory, including 
\cite{absil2005convergence, attouch2009convergence, 
attouch2013convergence, noll2014convergence, frankel2015splitting, 
bolte2016majorization}, exploit the {\L}ojasiewicz property 
to prove the convergence of various optimization algorithms.
By applying abstract convergence theorems such as 
\cite[Theorem 3.2]{attouch2013convergence}
and \cite[Theorem 3.1]{frankel2015splitting} 
to the MS algorithm, we obtain the following theorem:%
\footnote{%
As we have observed in Section~\ref{sec:Target}
that the MS algorithm is an example of the MM algorithm,
we might alternatively be able to apply abstract convergence theorems
for the MM algorithm~\cite{bolte2016majorization} to the MS algorithm.
Although convergence of the MS algorithm could be proved in this way,
the resulting bound of the convergence rate can become looser than 
that given by Theorems~\ref{thm:MS-Rate} and \ref{thm:MS-Upper} in this paper.
This is because that bound depends on the {\L}ojasiewicz exponent 
of the function $\bar{f}(\bx+\bm{m}(\bx)|\bx)$ (called the value function) 
introduced in \cite{bolte2016majorization} (not of the KDE), which 
is in general flatter than the KDE at the critical point.}
\begin{theorem}[Convergence guarantee]
\label{thm:MS-GCG}
Assume Assumptions~\ref{asm:RS} and \ref{asm:QM}. 
Let $(\by_t)_{t\in\bbN}$ be the mode estimate sequence
obtained by the MS algorithm~\eqref{eq:MS-Iter} starting from $\by_1$
with $f(\by_1)>0$.
Assume further, for the closure $\cl(\Conv(\{\by_t\}_{t\ge\tau}))$ of 
the convex hull $\Conv(\{\by_t\}_{t\ge\tau})$
of $\{\by_t\}_{t\ge\tau}$ with some $\tau\in\bbN$, that 
\begin{itemize}
\item[\hyt{a1}]
the KDE $f$ is differentiable and has a Lipschitz-continuous gradient on $\cl(\Conv(\{\by_t\}_{t\ge\tau}))$
(i.e., there exists a constant $L\ge0$ such that $\|\nabla f(\bx)-\nabla f(\bx')\|\le L\|\bx-\bx'\|$ 
for any $\bx, \bx'\in\cl(\Conv(\{\by_t\}_{t\ge\tau}))$, 
where the minimum of such a constant $L$ is called the Lipschitz constant of $\nabla f$ on $\cl(\Conv(\{\by_t\}_{t\ge\tau}))$), and
\item[\hyt{a2}]
the KDE $f$ has the {\L}ojasiewicz property on $\cl(\Conv(\{\by_t\}_{t\ge\tau}))$.
\end{itemize}
Then, the mode estimate sequence $(\by_t)_{t\in\bbN}$ 
has a finite-length trajectory (i.e., $\sum_{t=1}^\infty\|\by_{t+1}-\by_t\|<\infty$)
and converges to a critical point $\bar{\by}$ of the KDE $f$.
\end{theorem}

We next argue how one can replace the assumptions~\hyl{a1} and \hyl{a2}
on the KDE $f$ to assumptions on the kernel $K$
in such a way that the latter ones provide sufficient conditions
for the former ones.

Let us focus on the assumption~\hyl{a1} of Theorem~\ref{thm:MS-GCG} first. 
If a kernel $K$ is differentiable with a Lipschitz-continuous gradient,
then the KDE $f$ using the kernel $K$ trivially 
satisfies the assumption~\hyl{a1} for any $\tau$,
simply because the summation of functions preserves
the differentiability, as well as the Lipschitz continuity of the gradients. 
Therefore, for the convergence guarantee of the mode estimate sequence $(\by_t)_{t\in\bbN}$,
we can replace the assumption~\hyl{a1} on the KDE $f$
with the following assumption on the kernel $K$: 
\begin{assumption}
\label{asm:LCG}%
The kernel $K$ is differentiable and has a Lipschitz-continuous gradient.
\end{assumption}
Note that Assumption~\ref{asm:LCG} also implies that the kernel $K$ is of class $C^1$.

We next argue how one can replace the assumption~\hyl{a2}
of Theorem~\ref{thm:MS-GCG} with an assumption on the kernel $K$. 
According to Propositions \ref{prop:SA} and \ref{prop:GSprop},
when the kernel $K$ is analytic or $C^1$ globally subanalytic, 
it is clear that the corresponding KDE $f$
is so as well and has the {\L}ojasiewicz property.
We argue in the following that requiring the kernel $K$
to be $C^1$ subanalytic is indeed enough
in order for the assumption~\hyl{a2} to hold: 
Under Assumptions~\ref{asm:RS} and \ref{asm:QM}, as well as
the condition $f(\by_1)>0$, 
the mode estimate $\by_t$ for $t\ge2$ becomes
a convex combination of the data points $\{\bx_i\}_{i\in[n]}$,
that is, 
a weighted mean of $\{\bx_i\}_{i\in[n]}$ with non-negative weights,
and thus it lies in the convex hull $\Conv(\{\bx_i\}_{i\in[n]})$ of $\{\bx_i\}_{i\in[n]}$,
which is a bounded set.
Therefore, we can restrict the domain of 
every kernel $K(\frac{\cdot-\bx_i}{h})$, $i=1,\ldots,n$, 
to $\Conv(\{\bx_i\}_{i\in[n]})$ without any problems.
Also, every kernel $K(\frac{\cdot-\bx_i}{h})$ is 
bounded under Assumption~\ref{asm:RS}.
Therefore, when the kernel $K$ is $C^1$ subanalytic,
the restriction of $K(\frac{\cdot-\bx_i}{h})$ to $\Conv(\{\bx_i\}_{i\in[n]})$ 
becomes a $C^1$ subanalytic function with a bounded graph, 
and consequently, it is $C^1$ globally subanalytic due to Proposition~\ref{prop:GSprop}. 
Hence, the restriction of the corresponding KDE is also $C^1$ globally subanalytic
due to Proposition~\ref{prop:GSprop} and has the {\L}ojasiewicz property 
due to Proposition~\ref{prop:SA}.
Given this consideration, we do not have to require global subanalyticity, 
and requiring $C^1$ subanalyticity to the kernel $K$ is 
sufficient for the assumption~\hyl{a2} to be satisfied for any $\tau$.
Therefore, under Assumptions~\ref{asm:RS}, \ref{asm:QM}, 
and \ref{asm:LCG} and the condition $f(\by_1)>0$, we can
replace the assumption~\hyl{a2} on the KDE $f$ with 
the following assumption on the kernel $K$: 
\begin{assumption}
\label{asm:LP}%
The kernel $K$ is analytic or subanalytic.
\end{assumption}

Consequently, the following theorem will be obtained 
as a direct corollary of Theorem~\ref{thm:MS-GCG},
which assures the convergence independently of
the mode estimate sequence $(\by_t)_{t\in\bbN}$. 
\begin{theorem}[{Corollary of Theorem~\ref{thm:MS-GCG}}]
\label{thm:MS-CG}
Assume Assumptions~\ref{asm:RS}, \ref{asm:QM}, \ref{asm:LCG},
and \ref{asm:LP}.
Let $(\by_t)_{t\in\bbN}$ be the mode estimate sequence
obtained by the MS algorithm~\eqref{eq:MS-Iter} starting from $\by_1$
with $f(\by_1)>0$.
Then, the mode estimate sequence $(\by_t)_{t\in\bbN}$ 
has a finite-length trajectory and converges to a critical point $\bar{\by}$ of the KDE $f$.
\end{theorem}

The main significance of Theorem~\ref{thm:MS-CG} is 
that it reveals for the first time the convergence of 
the MS algorithm for several piecewise polynomial kernels
including the biweight and triweight kernels.
In particular, the biweight kernel is known to be 
optimal among non-negative kernels 
in terms of the asymptotic statistical efficiency 
for the KDE-based mode estimation 
\cite{parzen1962estimation, yamasaki2023optimal}.
%
More concretely, for a mode of the true probability density function
with a non-degenerate Hessian at the mode, 
the main term of the asymptotic mean squared error of 
the 1-dimensional KDE-based mode estimator using a non-negative kernel $K$ 
and optimal bandwidth parameter for that kernel
is proportional to the kernel-dependent term
$(\int_{-\infty}^\infty u^2K(u)\,du)^{\frac{6}{7}}\cdot(\int_{-\infty}^\infty \{K'(u)\}^2\,du)^{\frac{4}{7}}$
(we call its inverse the asymptotic statistical efficiency), 
and \cite{granovsky1991optimizing} showed that 
the biweight kernel minimizes this kernel-dependent term.
Moreover, \cite{yamasaki2023optimal} obtained 
similar results for the multi-dimensional case.
The triweight kernel is also relatively good in the same perspective; 
see Table~\ref{tab:Kernel} where we arrange kernels 
in the order of the asymptotic statistical efficiency for the 1-dimensional case 
(calculated ignoring a finite number of non-differentiable points) from the top.%
\footnote{%
\cite{epanechnikov1969non, granovsky1989optimality} show that 
the Epanechnikov kernel minimizes the asymptotic mean integrated squared error of the KDE 
using the associated optimal bandwidth parameter among non-negative kernels.
It should be noted, however, that, although this fact was mentioned in papers
which study convergence properties of the MS algorithm,
such as \cite{comaniciu1999mean} and \cite{Huang2018},
it does not imply the optimality of the Epanechnikov kernel
for the KDE-based mode estimation, a representative application of the MS algorithm, in any sense.}

\begin{table*}[!t]
\centering%
\renewcommand{\tabcolsep}{3pt}
\caption{%
Fulfillment of assumptions of kernels (satisfying Asms.\,\ref{asm:RS} and \ref{asm:LP}),
and presence of convergence guarantee, convergence rate evaluation, 
and worst-case bound of convergence rate of the mode estimate sequence 
$(\bm{y}_t)_{t\in\mathbb{N}}$ obtained by the corresponding MS algorithm, 
where $\triangle$ implies that it holds under additional conditions.
The table also lists the reference numbers or the theorem in this paper
where each result is first proven.}
\label{tab:Kernel}
\begin{tabular}{c|c|c|c|cc|cc|cc}
\toprule
 \multirow{2}{*}{Kernel} & \multirow{2}{*}{$\hat{K}(u)\propto$} & 
 \multirow{2}{*}{Asm.\,\ref{asm:QM}} & \multirow{2}{*}{Asm.\,\ref{asm:LCG}} & 
\multicolumn{2}{c}{Convergence} &
 \multicolumn{2}{|c}{Convergence} & 
\multicolumn{2}{|c}{Worst-case bound}\\
&&&& 
\multicolumn{2}{c}{guarantee} &
 \multicolumn{2}{|c}{rate evaluation} & 
\multicolumn{2}{|c}{of convergence rate}\\
\midrule
Biweight & $\{(1-u)_+\}^2$ & \checkmark & \checkmark &
{\checkmark} & Thm.\;\ref{thm:MS-CG}&
{\checkmark} & Thm.\;\ref{thm:MS-Rate}&
{\checkmark} & Thm.\;\ref{thm:MS-Upper}\\
-- & $\{(1-u)_+\}^{3/2}$ & \checkmark & $\times$ &
$\triangle$ & Thm.\;\ref{thm:MS-GCG} under \hyl{a1}\&\hyl{a2}&
$\triangle$ & Thm.\;\ref{thm:MS-Rate} under \hyl{a1}\&\hyl{a2}&
{$\times$}& \\
Triweight & $\{(1-u)_+\}^3$ & 
\checkmark & \checkmark & 
{\checkmark} & Thm.\;\ref{thm:MS-CG}&
{\checkmark} & Thm.\;\ref{thm:MS-Rate}&
{\checkmark} & Thm.\;\ref{thm:MS-Upper}\\
Tricube & $\{(1-u^{3/2})_+\}^3$ & 
$\times$ & \checkmark & 
{$\times$} & & {$\times$} & & {$\times$} & \\
Cosine & $\cos(\frac{\pi u^{1/2}}{2})\bbI(u\le1)$ & 
\checkmark & $\times$ & 
$\triangle$ & Thm.\;\ref{thm:MS-GCG} under \hyl{a1}\&\hyl{a2}&
$\triangle$ & Thm.\;\ref{thm:MS-Rate} under \hyl{a1}\&\hyl{a2}&
{$\times$}& \\
Epanechnikov & $(1-u)_+$ & 
\checkmark & $\times$ & 
{\checkmark} & \cite{comaniciu1999mean, Huang2018} & 
{\checkmark} & \cite{comaniciu1999mean, Huang2018} & 
{\checkmark} & \cite{comaniciu1999mean, Huang2018}\\
Gaussian & $e^{-u}$ & \checkmark & \checkmark &
{\checkmark} & \cite{yamasaki2019ms} & {\checkmark} & Thm.\;\ref{thm:MS-Rate}& {$\times$}& \\
Logistic & $\frac{1}{e^{u^{1/2}}+2+e^{-u^{1/2}}}$ &
\checkmark & \checkmark & 
{\checkmark} & \cite{yamasaki2019ms} & {\checkmark} & Thm.\;\ref{thm:MS-Rate}& {$\times$}& \\
Cauchy & $\frac{1}{1+u}$ & 
\checkmark& \checkmark & 
{\checkmark} & \cite{yamasaki2019ms} & {\checkmark} & Thm.\;\ref{thm:MS-Rate}& {$\times$}& \\
\bottomrule
\end{tabular}
\end{table*}
\subsection{Convergence Rate}
\label{sec:Rate}
In this subsection, we study convergence rate of the MS algorithm.
As mentioned at the end of Section~\ref{sec:Related},
there are only a few studies on the convergence rate of the MS algorithm:
It was proved in~\cite{comaniciu1999mean} and \cite{Huang2018}
that the MS algorithm with the Epanechnikov kernel
converges in a finite number of iterations,
and in~\cite{carreira2007gaussian}
that the MS algorithm with the Gaussian kernel
exhibits linear convergence provided that
the Hessian of the KDE at a critical point is non-degenerate.
We here establish a convergence rate evaluation for 
other kernels under more general situations.

Assume for a moment 
that the kernel $K$ is twice continuously differentiable 
(i.e., $K$ is of class $C^2$) and hence the KDE $f$ is so as well,
in addition to Assumptions~\ref{asm:RS} and \ref{asm:QM}.
Consider Taylor expansion of the map 
$\by_t\mapsto\by_{t+1}=\by_t+\bm{m}(\by_t)$ 
around a critical point $\bar{\by}$ of $f$,
\begin{align}
\label{eq:Tay-Exp1}
	\by_{t+1}
	=\bar{\by}+\bfJ(\bar{\by})(\by_t-\bar{\by})+o(\|\by_t-\bar{\by}\|),
\end{align}
where $\bfJ(\bar{\by})$ is the Jacobian of 
the map $\bx\mapsto\bx+\bm{m}(\bx)$ at $\bx=\bar{\by}$.
When $\by_t$ is sufficiently close to $\bar{\by}$,
one has the relation
\begin{align}
\label{eq:Tay-Exp2}
	\|\by_{t+1}-\bar{\by}\|
	\le\|\bfJ(\bar{\by})(\by_t-\bar{\by})\|
	+\epsilon\|\by_t-\bar{\by}\|
%
%
\end{align}
with a sufficiently small $\epsilon>0$.
This relation suggests that the mode estimate sequence 
$(\by_t)_{t\in\bbN}$ achieves the linear convergence 
(i.e., $\|\by_{t+1}-\bar{\by}\|\le(|\delta|+\epsilon)\|\by_t-\bar{\by}\|$ for sufficiently large $t$)
when the matrix $\bfJ(\bar{\by})$ is real symmetric and 
the farthest-from-zero eigenvalue $\delta$ of $\bfJ(\bar{\by})$ has the absolute value less than 1.

Simple calculation reveals that the Jacobian $\bfJ(\bar{\by})$
of the map $\bx\mapsto\bx+\bm{m}(\bx)$ at $\bx=\bar{\by}$ is given by 
\begin{align}
\label{eq:Jac1}
	\bfJ(\bar{\by})
	=\frac{\sum_{i=1}^n\hat{K}''(\|\frac{\bar{\by}-\bx_i}{h}\|^2/2)(\bx_i-\bar{\by})(\bx_i-\bar{\by})^\top}
	{-h^2\sum_{i=1}^n\hat{K}'(\|\frac{\bar{\by}-\bx_i}{h}\|^2/2)},
\end{align}
which is real symmetric.
It should be noted that the denominator
of the right-hand side of~\eqref{eq:Jac1}
is equal to $nh^{d+2}\check{f}(\bar{\by})\ge0$,
which is positive if $f(\bar{\by})>0$.
As Assumption~\ref{asm:QM} ensures that $\hat{K}''$ is non-negative, 
$\bfJ(\bar{\by})$ becomes positive semidefinite.
On the other hand, 
from $\bm{m}(\bx)=\frac{h^2}{\check{f}(\bx)}\nabla f(\bx)$ 
and $\nabla f(\bar{\by})=\bm{0}$,
the Jacobian is also calculated as 
\begin{align}
\label{eq:Jac2}
	\bfJ(\bar{\by})
	=\bfI_d+\frac{h^2}{\check{f}(\bar{\by})}\nabla^2 f(\bar{\by}),
\end{align}
where $\bfI_d$ is the $d\times d$-identity matrix.
The fact that $\nabla^2f$ at a local maximizer becomes negative semidefinite, 
together with the positive semidefiniteness of the Jacobian $\bfJ(\bar{\by})$ mentioned above, 
implies that $\bfJ(\bar{\by})$ at a local maximizer $\bar{\by}$
of $f$ has eigenvalues within the interval $[0,1]$.
The following proposition, which is a generalization 
of \cite{carreira2007gaussian} with the Gaussian kernel
to that with a generic 
kernel allowing twice continuous differentiability of the KDE at $\bar{\bm{y}}$,
shows the linear convergence 
when the Hessian $\nabla^2f(\bar{\by})$ is non-degenerate.

\begin{proposition}[Linear convergence in non-degenerate case]
\label{prop:LC}
Assume Assumptions~\ref{asm:RS} and \ref{asm:QM}, 
that the mode estimate sequence $(\by_t)_{t\in\bbN}$
obtained by the MS algorithm~\eqref{eq:MS-Iter} 
converges to $\bar{\bm{y}}=\lim_{t\to\infty}\bm{y}_t$, 
that there exists a neighborhood of $\bar{\by}$ 
where the KDE $f$ is twice continuously differentiable 
(which holds when the kernel $K$ is of class $C^2$),
and that the Hessian $\nabla^2f(\bar{\by})$ of $f$
at $\bar{\by}$ is negative definite.
Then, the mode estimate sequence $(\by_t)_{t\in\bbN}$ 
achieves the linear convergence: 
for the largest eigenvalue $\lambda\in[-\frac{\check{f}(\bar{\by})}{h^2},0)$ of $\nabla^2f(\bar{\by})$
and any $\epsilon\in[0,-\frac{h^2}{\check{f}(\bar{\by})}\lambda)$,
there exists $\tau\in\bbN$ such that 
$\|\by_{t+1}-\bar{\by}\|\le q\|\by_t-\bar{\by}\|$ for any $t\ge\tau$
with $q=1+\frac{h^2}{\check{f}(\bar{\by})}\lambda+\epsilon\in[0,1)$.
%
\end{proposition}

Proposition~\ref{prop:LC} tells us the typical convergence rate of the MS algorithm.
Additionally, we would like to note that the linear convergence guarantee in this proposition
implies the exponential rate convergence as well
(although the converse does not hold in general):
Applying the relation $\|\by_{t+1}-\bar{\by}\|\le q\|\by_t-\bar{\by}\|$ recursively yields 
$\|\by_t-\bar{\by}\|\le q^{t-\tau}\|\by_\tau-\bar{\by}\|$ for $t\ge\tau$,
which implies $\|\by_t-\bar{\by}\|=O(q^t)$.
Also, the second-order Taylor expansion of the KDE $f$ around the critical point $\bar{\by}$ 
shows the exponential-rate convergence of the density estimate sequence $(f(\by_t))_{t\in\bbN}$ as 
$|f(\bar{\by})-f(\by_t)|\approx|(\by_t-\bar{\by})^\top\{\nabla^2 f(\bar{\by})\}(\by_t-\bar{\by})|=O(q^{2t})$.

In the above proposition, we excluded from our consideration 
the case where the Hessian $\nabla^2f(\bar{\by})$ is degenerate. 
When the Hessian is degenerate,
the Jacobian $\bfJ(\bar{\by})$ has the largest eigenvalue equal to 1, 
and then analysis based on the first-order Taylor approximation of $\bm{m}(\by_t)$
does not lead to the (linear) convergence of the MS algorithm. 
In order to evaluate convergence rate along the same line of the analysis
in such cases,
one might have to investigate effects of the residual term in more detail.

Discussion based on the {\L}ojasiewicz property allows us to derive
convergence rate of the MS algorithm
under a weaker assumption on differentiability. 
More concretely, by applying \cite[Theorem 3.5]{frankel2015splitting}, 
we can prove the following theorem on the convergence rate
of the MS algorithm 
that covers more general kernels and the degenerate case as well.
It provides upper bounds of the convergence rate
determined by the {\L}ojasiewicz exponent $\theta$ of the KDE. 

\begin{theorem}[Convergence rate evaluation]
\label{thm:MS-Rate}
Under the same assumptions as in Theorem~\ref{thm:MS-GCG} or \ref{thm:MS-CG}, 
assume further that the KDE $f$ has the {\L}ojasiewicz exponent $\theta$ at $\bar{\by}=\lim_{t\to\infty}\by_t$, 
for the mode estimate sequence $(\by_t)_{t\in\bbN}$ obtained by the MS algorithm~\eqref{eq:MS-Iter}.
Then, one has that
\begin{itemize}
\item[\hyt{b1}]%
if $\theta\in[0,\frac{1}{2})$ 
then the MS algorithm converges in a finite number of iterations
(there exists $\tau\in\bbN$ such that 
$\by_t=\bar{\by}_\tau$ and $f(\by_t)=f(\by_\tau)$ for any $t\ge\tau$),
\item[\hyt{b2}]%
if $\theta=\frac{1}{2}$
then the MS algorithm achieves the exponential-rate convergence
(there exists $q\in(0,1)$ such that 
$\|\by_t-\bar{\by}\|=O(q^t)$ and $f(\bar{\by})-f(\by_t)=O(q^{2t})$), or
\item[\hyt{b3}]%
if $\theta\in(\frac{1}{2},1)$
then the MS algorithm achieves the polynomial-rate convergence
($\|\by_t-\bar{\by}\|=O(t^{-\frac{1-\theta}{2\theta-1}})$
and $f(\bar{\by})-f(\by_t)=O(t^{-\frac{1}{2\theta-1}})$).
\end{itemize}
\end{theorem}

\begin{figure*}[!t]
\centering
\begin{tabular}{c}
\begin{overpic}[height=2.75cm, bb=0 0 1698 294]{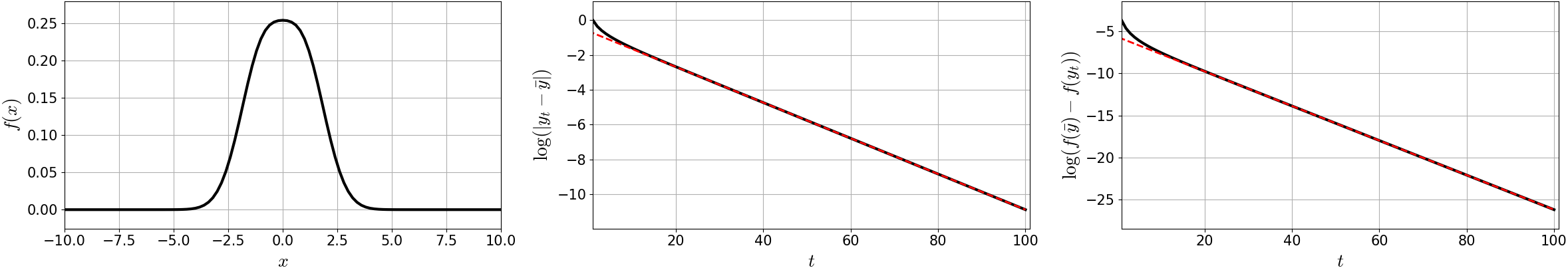}
\put(-2,16){{\scriptsize\hyt{i}}}\end{overpic}\\
\begin{overpic}[height=2.75cm, bb=0 0 1698 294]{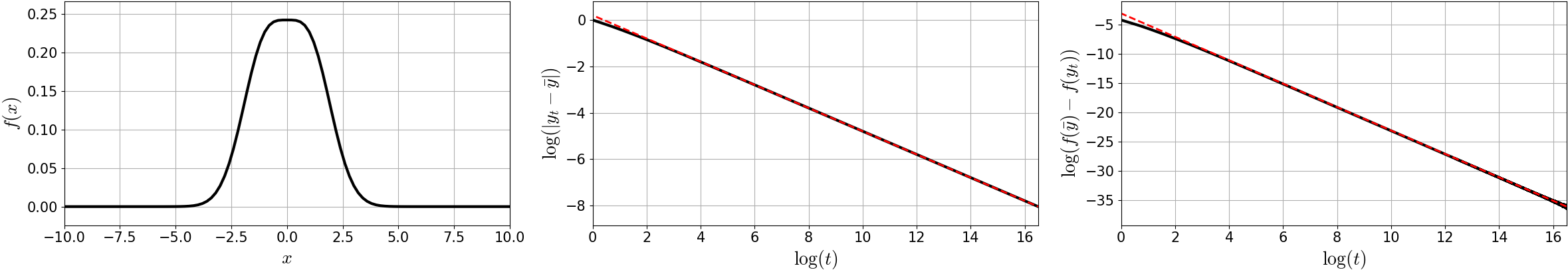}
\put(-2,16){{\scriptsize\hyt{ii}}}\end{overpic}\\
\begin{overpic}[height=2.75cm, bb=0 0 1698 294]{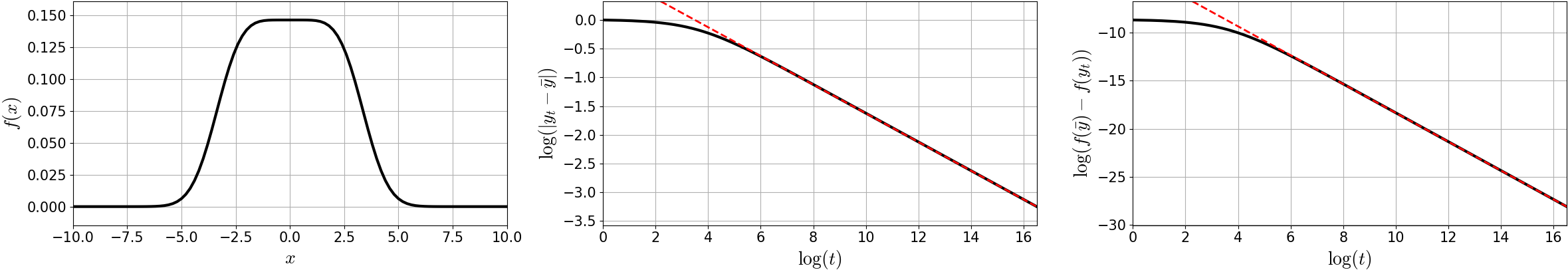}
\put(-2,16){{\scriptsize\hyt{iii}}}\end{overpic}\\
\begin{overpic}[height=2.75cm, bb=0 0 1698 294]{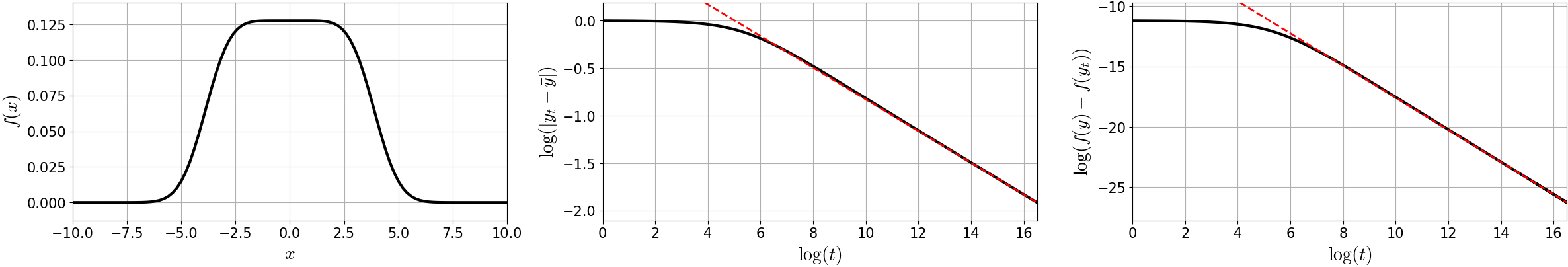}
\put(-2,16){{\scriptsize\hyt{iv}}}\end{overpic}
\end{tabular}
\caption{
Instances of the KDE, and plots of $|y_t-\bar{y}|$ and $f(\bar{y})-f(y_t)$ versus $t$
with $(y_t)_{t\in\bbN}$ obtained by the MS algorithm with the Gaussian kernel
with $d=1$ and $h=1$.
For every case, the mode $\bar{y}$, to which 
$(y_t)_{t\in\bbN}$ converges, is the origin (i.e., $\bar{y}=0$).
\protect\hyl{i} $n=2$, $x_1,x_2=\pm0.95$, 
with which the second derivative $f^{(2)}(x)$ of the KDE $f(x)$ is non-degenerate at the mode $x=\bar{y}$, 
yielding $\theta=1/2$. 
The plots of $|y_t-\bar{y}|$ and $f(\bar{y})-f(y_t)$ are shown in semilog plots.
\protect\hyl{ii} $n=2$, $x_1,x_2=\pm1$, 
with which $f^{(i)}(\bar{y})=0$ for all $i\in[3]$ and $f^{(4)}(\bar{y})<0$,
yielding $\theta=1-1/4=3/4$. 
\protect\hyl{iii} $n=6$, 
$x_1,x_2=\pm0.564\ldots,x_3,x_4=\pm1.721\ldots,x_5,x_6=\pm2.801\ldots$, 
which were carefully chosen so that 
$f^{(i)}(\bar{y})=0$ for all $i\in[5]$ and $f^{(6)}(\bar{y})<0$, yielding $\theta=1-1/6=5/6$. 
\protect\hyl{iv} $n=6$, 
$x_1,x_2=\pm0.651\ldots,x_3,x_4=\pm1.959\ldots,x_5,x_6=\pm3.243\ldots$, 
which were carefully chosen so that 
$f^{(i)}(\bar{y})=0$ for all $i\in[7]$ and $f^{(8)}(\bar{y})<0$, yielding $\theta=1-1/8=7/8$. 
The plots of $|y_t-\bar{y}|$ and $f(\bar{y})-f(y_t)$ in 
\protect\hyl{ii}, \protect\hyl{iii}, and \protect\hyl{iv} are shown in log-log plots.
Simulation results are shown as black solid curves,
and the red dotted lines show the asymptotic convergence rates
predicted by Proposition~\ref{prop:LC} and 
Theorem~\ref{thm:MS-Rate}.}
\label{fig:LR}
\end{figure*}

We would like to mention that, among the three cases appearing
in Theorem 3, the case $\theta\in[0,\frac{1}{2})$ may happen only exceptionally.
For example, when the convergent point $\bar{\bm{y}}$ is a local minimizer of the KDE, 
the {\L}ojasiewicz exponent $\theta$ at $\bar{\bm{y}}$ becomes 0,
implying that convergence to a local minimizer should happen
in a finite number of iterations. 
%
On the other hand, $\theta\in(0,\frac{1}{2})$ would not hold typically 
under the assumptions of Theorem~\ref{thm:MS-Rate}:
Consider the case where the KDE $f(\by)$ behaves like \eqref{eq:ExpA}, i.e., 
$f(\by)=f(\bar{\by})-\|\by-\bar{\by}\|^\alpha$, locally around a mode $\by=\bar{\by}$.
Differentiating both sides of this local equality and Lipschitz continuity of $\nabla f$ 
(the assumption~\hyl{a1} or Assumption~\ref{asm:LCG}) show
$|\nabla f(\by)-\nabla f(\bar{\by})|=\alpha\|\by-\bar{\by}\|^{\alpha-1}
\le L\|\by-\bar{\by}\|$ with the Lipschitz constant $L\ge0$ of $\nabla f$.
This implies $\alpha\ge2$ and hence $\theta\ge\frac{1}{2}$.
Therefore, when the mode estimate sequence converges to a mode as expected 
for the MS algorithm, typically Theorem~\ref{thm:MS-Rate} \hyl{b2} or \hyl{b3} 
tells us the convergence rate.

It should be noted that
the Epanechnikov kernel does not satisfy Assumption~\ref{asm:LCG} as shown in Table~\ref{tab:Kernel},
so that Theorem~\ref{thm:MS-Rate}
will be applicable to the MS algorithm with the Epanechnikov kernel
only under the conditions where the assumptions~\hyl{a1} and \hyl{a2}
in Theorem~\ref{thm:MS-GCG} are satisfied. 
With the Epanechnikov kernel, 
the {\L}ojasiewicz exponent at the mode of the KDE is typically $\frac{1}{2}$, 
and if applying Theorem~\ref{thm:MS-Rate} is legitimate, it 
suggests the exponential-rate convergence via \hyl{b2},
which is a looser evaluation than 
the finite-time convergence guaranteed by \cite{comaniciu1999mean} and \cite{Huang2018}.
However, the convergence rate evaluation provided by Theorem~\ref{thm:MS-Rate} \hyl{b2} and \hyl{b3} seems to be
almost tight in other generic cases, as demonstrated in Figure~\ref{fig:LR},
where the behaviors of the MS algorithm with the Gaussian kernel
in the one-dimensional case are shown, with carefully chosen positions
of data points so that the KDE has a degenerate Hessian at its mode.

Theorem~\ref{thm:MS-Rate}, as well as the experimental results
summarized in Figure~\ref{fig:LR}, strongly suggests that
the {\L}ojasiewicz exponent of the KDE bears essential information
about the convergence rate of the MS algorithm.
It is known, however, 
that the calculation of the {\L}ojasiewicz exponent is difficult in general 
(see discussion of \cite{li2018calculus} for details).
Even in such a circumstance, 
\cite{gwozdziewicz1999lojasiewicz, d2005explicit, kurdyka2014separation} 
provided bounds of the {\L}ojasiewicz exponent for polynomial functions.
On the ground of \cite[Proposition 4.3]{d2005explicit}, we can provide an upper bound of 
the {\L}ojasiewicz exponent of the KDE with a piecewise polynomial kernel.

\begin{theorem}[Bound of {\L}ojasiewicz exponent]
\label{thm:MS-Upper}
Assume that the kernel $K$ is of class $C^1$ and piecewise polynomial with maximum degree $k\ge2$.
Then, the {\L}ojasiewicz exponent $\theta$
of the KDE $f$ at any critical point $\bar{\by}$ is bounded from above as 
\begin{align}
\label{eq:Loja-exp-UB}
	\theta\le 1-\frac{1}{\max\{k(3k-4)^{d-1}, 2k(3k-3)^{d-2}\}}, 
\end{align}
provided that $f$ is not constant in any subdomain
with a non-empty intersection with the $\epsilon$-neighborhood of $\bar{\by}$
for any $\epsilon>0$. 
\end{theorem}

This bound of the {\L}ojasiewicz exponent, 
together with Theorem~\ref{thm:MS-Rate} \hyl{b3}, 
gives a worst-case bound of the convergence rate 
of the MS algorithm using a piecewise polynomial kernel.
However, it should be noted that the bound provided in Theorem~\ref{thm:MS-Upper} 
is not tight in general and might be improved by future research.

Finally, we would like to make a few remarks regarding the discussion in this section:
First, the {\L}ojasiewicz exponent $\theta$, which appears in 
Theorems~\ref{thm:MS-Rate} and \ref{thm:MS-Upper}, 
is the one at a critical point of the KDE \eqref{eq:KDE} and depends
not only on the kernel $K$ but also on 
the data points $\{\bm{x}_i\}_{i=1}^n$ and bandwidth $h$,
so that our results on the convergence rate are not readily applicable
to the issue of how to select the kernel used in the MS algorithm. 
%
%
Kernel selection should also be affected by factors other than the convergence rate, 
such as quality of the output of the algorithm, 
like the optimality of the biweight kernel in terms of the asymptotic statistical efficiency 
for the KDE-based estimation of a non-degenerate mode
\cite{parzen1962estimation, granovsky1991optimizing, yamasaki2023optimal}.

\section{Conclusion and Future Work}
\label{sec:Concl}
We have shown that the mode estimate sequence generated 
by the MS algorithm using a $C^1$ subanalytic kernel 
converges to a critical point of the KDE (Theorem~\ref{thm:MS-CG}).
Our proof does neither presume 
that the KDE has a finite number of critical points or they are isolated, 
nor that its Hessian at a convergent point is non-degenerate, 
nor restriction on the size of the bandwidth or on the data dimension;
it utilizes the {\L}ojasiewicz property of the KDE.
The class of kernels covered by this theorem includes several piecewise polynomial kernels, 
such as the biweight kernel which is optimal among non-negative kernels 
for the KDE-based estimation of a non-degenerate mode
in terms of the asymptotic statistical efficiency 
\cite{granovsky1991optimizing, yamasaki2023optimal}.
The convergence analysis results in this paper extend the existing ones 
for the Epanechnikov kernel \cite{comaniciu1999mean, Huang2018} 
and for analytic kernels \cite{yamasaki2019ms}. 
Moreover, we not only provide a sufficient condition for 
the mode estimate sequence to achieve the linear convergence 
when the Hessian of the KDE at a convergent point
is non-degenerate (Proposition~\ref{prop:LC}), 
but also give a worst-case evaluation of 
the convergence rate (Theorems~\ref{thm:MS-Rate} and \ref{thm:MS-Upper})
that depends on the {\L}ojasiewicz exponent of the KDE
and is applicable even when the Hessian is degenerate.

The convergence theorems of the MS algorithm,
including ours for $C^1$ subanalytic kernels
and the existing ones for the Epanechnikov kernel and analytic kernels,
are also effective for the iteratively reweighted least squares algorithm, 
commonly used for various versions of robust M-type location estimation and regression 
\cite{huber1981robust, yamasaki2020kernel}.
Moreover, these results can be applied to several generalized MS algorithms.
The conditional MS algorithm, 
which is a representative estimation method for nonparametric modal regression 
\cite{hyndman1996estimating, einbeck2006modelling, chen2016nonparametric, sasaki2016modal},
can be regarded as a weighted version of the conventional MS algorithm with 
the weights determined by the values of the independent variable part of the data.
The convergence theorems can be generalized to 
the weighted version of the MS algorithm derived for the weighted objective function,
$\frac{1}{n h^d}\sum_{i=1}^n w_i K(\frac{\bx-\bx_i}{h})$ 
with constant weights $\{w_i\in(0,\infty)\}_{i\in[n]}$. 
Other instances of the generalized MS algorithms include
an MS variant derived for the KDE $\frac{1}{n}\sum_{i=1}^n\frac{1}{h_i^d}K(\frac{\bx-\bx_i}{h_i})$
with datapoint-wise bandwidths $\{h_i\in(0,\infty)\}_{i\in[n]}$ \cite{comaniciu2001variable, LiHuWu2007},
and the over-relaxation of the MS algorithm, 
$\by_{t+1}=\by_t+\zeta\bm{m}(\by_t)$ with a constant $\zeta\in(0,2)$ \cite{yamasaki2019ms}.
Even under these generalizations, 
a guarantee of the convergence to a critical point 
and a convergence rate evaluation still hold as well.

The subspace constrained MS algorithm \cite{ozertem2011locally, ghassabeh2013some}, 
another MS variant, is a method for estimating principal curves and principal surfaces 
as ridges of the KDE \cite{hastie1989principal, sasaki2017estimating}.
It iterates an update rule that is expected to converge to 
a point on a ridge of the KDE instead of its critical point. 
The convergence property of that algorithm would be related to that of the MS algorithm but is still open, 
and analysis with the {\L}ojasiewicz property might be useful for it.

\section*{Acknowledgment}
This work was supported by Grant-in-Aid for JSPS Fellows, Number 20J23367.
We would like to thank the authors of the article~\cite{arias2016estimation} for kindly 
drawing our attention to the errata \cite{Errata} that accompanies that article.

\bibliographystyle{IEEEtran}
\bibliography{bibtex}

\begin{IEEEbiography}[{\includegraphics[width=1in,height=1.25in,clip,bb=0 0 640 800]{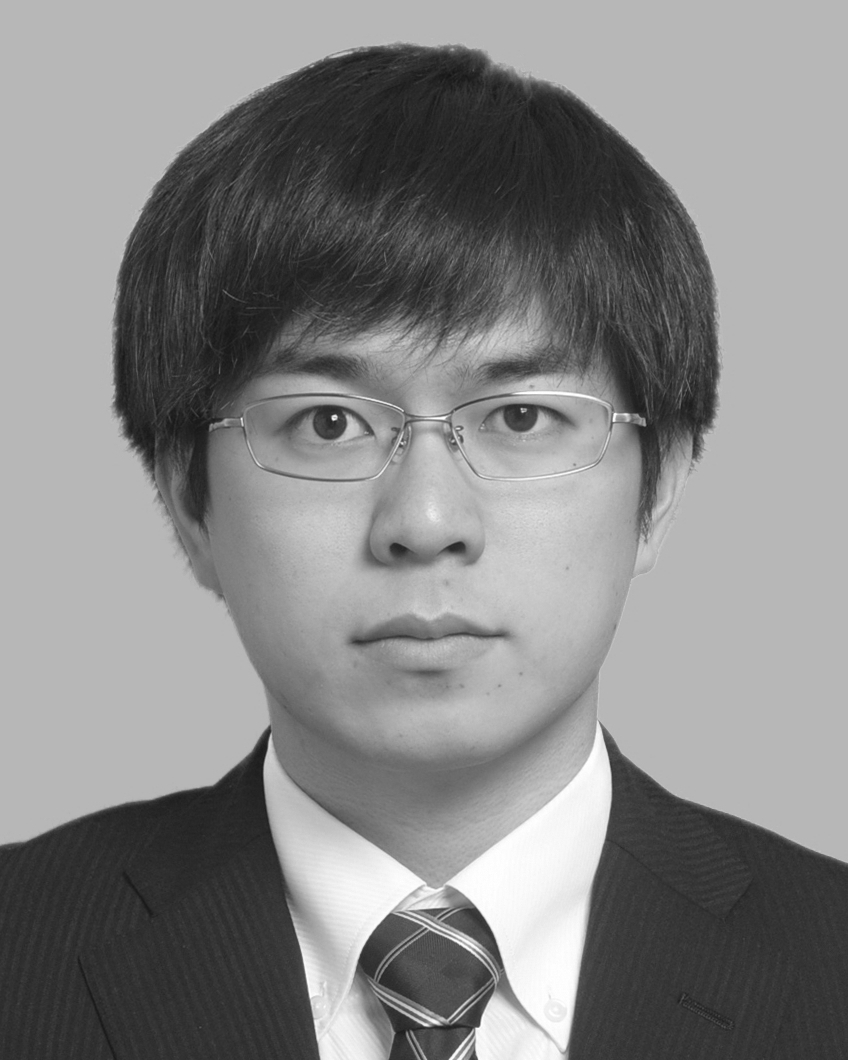}}]{Ryoya Yamasaki}
received the B.E.\ and M.Inf.\ degrees from Kyoto University, Kyoto, Japan, in 2018 and 2020, respectively.
He is currently working toward the D.Inf.\ degree of Graduate School of Informatics, Kyoto University, Kyoto, Japan.
His research interests are in areas of statistics and machine learning.
\end{IEEEbiography}
\begin{IEEEbiography}[{\includegraphics[width=1in,height=1.25in,clip,bb=0 0 800 1000]{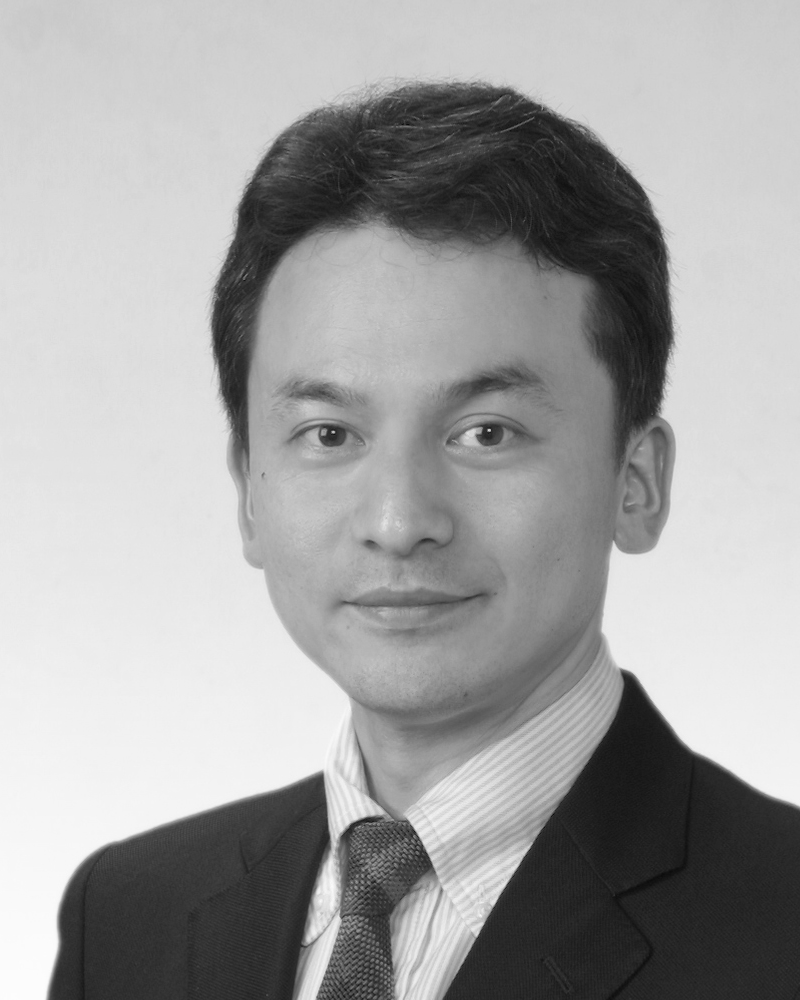}}]{Toshiyuki Tanaka} 
received the B.E., M.E., and D.E.\ degrees from the University of Tokyo, Tokyo, Japan, in 1988, 1990, and 1993, respectively.
He is currently a professor of Graduate School of Informatics, Kyoto University, Kyoto, Japan.
His research interests are in areas of information, coding, and communications theory, and statistical learning.
\end{IEEEbiography}

\clearpage
\onecolumn
\renewcommand{\theequation}{S\arabic{equation}}
\renewcommand{\thelemma}{S\arabic{lemma}}
\renewcommand{\theclaim}{S\arabic{claim}}
\setcounter{equation}{0}
\setcounter{lemma}{0}
\setcounter{claim}{0}

\section*{S1 Proofs of Theorems}
\label{sec:PT}
In this appendix, we provide proofs of the theoretical results stated in the main text of this paper.
Proposition~\ref{prop:dens} is shown as 
Theorem 1 of \cite{comaniciu2002mean}
and Theorem 1 of \cite{yamasaki2019ms}
(or can be proved from Lemma~\ref{lem:A} described below),
and Proposition~\ref{prop:SA} is given by 
\cite{lojasiewicz1965ensembles, kurdyka1994wf}.
Refer to the description in Sections~\ref{sec:Prelim} and \ref{sec:Rate} 
respectively for the proof of Propositions~\ref{prop:GSprop} and \ref{prop:LC}.
Also, Theorem~\ref{thm:MS-CG} is a corollary of Theorem~\ref{thm:MS-GCG}
as explained in Section~\ref{sec:SCP}. 
We here give proofs of the other results, 
Proposition~\ref{prop:fw} and Theorems~\ref{thm:MS-GCG}, \ref{thm:MS-Rate}, and \ref{thm:MS-Upper}.

\textbf{Technical Lemmas for Proposition \ref{prop:fw}}\\
First, we provide two technical lemmas:

\begin{lemma}
\label{lem:checkK}
Assume that a kernel $K$ satisfies Assumptions~\ref{asm:RS} and \ref{asm:QM}. 
Then, $\check{K}$ defined via~\eqref{eq:SD-Pro-Ker} is
non-negative, non-increasing, and bounded.
\end{lemma}
\begin{proof}[Proof of Lemma~\ref{lem:checkK}]
It has been proved just above the definition~\eqref{eq:SD-Pro-Ker}
of $\check{K}$ in the main text. 
\end{proof}

\begin{lemma}
\label{lem:hatcheckK}
Assume that a kernel $K$ satisfies Assumptions~\ref{asm:RS} and \ref{asm:QM}. 
Then, for any constant $C_1>0$
there exists a constant $C_2>0$
such that 
$\{u\in[0,\infty)\mid\hat{K}(u)\ge C_1\}\subseteq\{u\in[0,\infty)\mid\check{K}(u)\ge C_2\}$ holds. 
\end{lemma}
\begin{proof}[Proof of Lemma~\ref{lem:hatcheckK}]
Since the profile $\hat{K}$ is non-increasing (Assumption~\ref{asm:QM}),
one has $\hat{K}(0)=\max_{u\ge0}\hat{K}(u)$.
One also has $\hat{K}(0)\neq0$, 
since otherwise the kernel $K(\cdot)$ is equal to 0 identically, 
contradicting the assumption that $K$ is normalized (Assumption~\ref{asm:RS}).
Since $\hat{K}(0)>0$ and 
$\hat{K}$ is continuous (Assumption~\ref{asm:RS}) 
and non-increasing (Assumption~\ref{asm:QM}),
if $C_1>\hat{K}(0)$ the set $\{u\in[0,\infty)\mid\hat{K}(u)\ge C_1\}$ is empty
and the statement of the lemma trivially holds.
We therefore assume $0<C_1\le\hat{K}(0)$ in the following.
For any such $C_1$ 
one can let $[0,a]=\{u\in[0,\infty)\mid\hat{K}(u)\ge C_1\}$ with $a\in[0,\infty)$.
Here, the finiteness of $a$ comes from
the normalization condition of $K$ (Assumption~\ref{asm:RS}).

Lemma~\ref{lem:checkK} shows that 
$\check{K}$ is non-negative and non-increasing.
One then has that $\check{K}(u)>0$ for any $u\in[0,a]$. 
It is because if there exists $b\in[0,a]$ such that $\check{K}(b)=0$ 
then for any $u\ge b$ one has $\check{K}(u)=0$ 
and hence $\hat{K}(u)=\hat{K}(b)\ge C_1>0$, 
which contradicts the normalization condition of $K$ (Assumption~\ref{asm:RS}).
Letting $C_2=\min_{u\in[0,a]}\check{K}(u)=\check{K}(a)>0$,
one has $\check{K}(u)\ge C_2$ for any $u\in[0,a]$,
which proves $\{u\in[0,\infty)\mid\hat{K}(u)\ge C_1\}
=[0,a]\subseteq\{u\in[0,\infty)\mid\check{K}(u)\ge C_2\}$
to hold with that $C_2$. 
\end{proof}

\textbf{Proof of Proposition \ref{prop:fw}}\\
Proposition~\ref{prop:dens} and 
Lemmas~\ref{lem:checkK} and \ref{lem:hatcheckK}
lead to Proposition~\ref{prop:fw}, as in the following proof.

\begin{proof}[Proof of Proposition~\ref{prop:fw}]
The ascent property (Proposition~\ref{prop:dens}) implies
\begin{align}
\label{eq:S1}
	f(\by_t)
	=\frac{1}{n h^d}\sum_{i=1}^n K\biggl(\frac{\by_t-\bx_i}{h}\biggr)
	\ge f(\by_1).
\end{align}
Let $i_t$ be an index in $[n]$ satisfying 
\begin{align}
\label{eq:S2}
	K\biggl(\frac{\by_t-\bx_{i_t}}{h}\biggl)
	=\max_{i\in[n]}K\biggl(\frac{\by_t-\bx_i}{h}\biggr).
\end{align}
This definition and inequality \eqref{eq:S1} lead to
\begin{align}
\label{eq:S3}
	\hat{K}\biggl(\biggl\|\frac{\by_t-\bx_{i_t}}{h}\biggr\|^2\biggr/2\biggr)
	=K\biggl(\frac{\by_t-\bx_{i_t}}{h}\biggr)
	\ge\frac{1}{n}\sum_{i=1}^n K\biggl(\frac{\by_t-\bx_i}{h}\biggr)
	\ge h^d f(\by_1),
\end{align}
which implies that 
$u=\|\frac{\by_t-\bx_{i_t}}{h}\|^2\bigr/2$ for any $t\in\bbN$
is in the set $\{u\mid \hat{K}(u)\ge C_1\}$ with $C_1=h^df(\by_1)>0$. 
Lemma~\ref{lem:hatcheckK} then 
states that there exists a constant $C>0$ such that for any $t\in\bbN$
\begin{align}
\label{eq:S5}
	\check{K}\biggl(\biggl\|\frac{\by_t-\bx_{i_t}}{h}\biggr\|^2\biggr/2\biggr)
	\ge C
\end{align}
holds.
From Lemma~\ref{lem:checkK}, $\check{K}$ is non-negative. 
Using this fact and inequality \eqref{eq:S5}, one consequently has that
\begin{align}
	\check{f}(\by_t)
	=\frac{1}{n h^d}\sum_{i=1}^n\check{K}\biggl(\biggl\|\frac{\by_t-\bx_i}{h}\biggr\|^2\biggr/2\biggr)
	\ge\frac{1}{n h^d}\check{K}\biggl(\biggl\|\frac{\by_t-\bx_{i_t}}{h}\biggr\|^2\biggr/2\biggr)
	\ge\frac{C}{n h^d}.
\end{align}
This concludes the proof. 
\end{proof}

\textbf{Technical Lemmas for Theorems \ref{thm:MS-GCG} and \ref{thm:MS-Rate}}\\
We here provide three technical lemmas that introduce 
positive constants $\bar{a}$, $\bar{b}$, and $\bar{c}$,
each of which defines a separate inequality:

\begin{lemma}[Sufficient increase condition]
\label{lem:A}
Assume Assumptions~\ref{asm:RS} and \ref{asm:QM}, and $f(\by_1)>0$.
Then, there exists $\bar{a}>0$ such that
\begin{align}
\label{eq:lem-a}
	f(\by_{t+1})-f(\by_t)\ge\bar{a}\|\by_{t+1}-\by_t\|^2
\end{align}
holds for any $t\in\bbN$.
\end{lemma}
\begin{proof}[{Proof of Lemma~\ref{lem:A}}]
Considering the coefficients of $\bx$- and $\|\bx\|^2$-dependent terms 
of the minorizer $\bar{f}(\bx|\by_t)$ in \eqref{eq:Mino-KDE1}, 
and the update rule of the MS algorithm \eqref{eq:MS-Iter},
one can find another representation of $\bar{f}(\bx|\by_t)$:
\begin{align}
	\bar{f}(\bx|\by_t)
	=-\frac{\check{f}(\by_t)}{2h^2}\|\by_{t+1}-\bx\|^2
	+\text{($\bx$-independent constant)}.
\end{align}
This representation, together with the ascent property
$f(\by_t)=\bar{f}(\by_t|\by_t)\le\bar{f}(\by_{t+1}|\by_t)\le f(\by_{t+1})$, 
yields the inequality 
\begin{align}
	f(\by_{t+1})-f(\by_t)
	\ge\bar{f}(\by_{t+1}|\by_t)-\bar{f}(\by_t|\by_t)
	=\frac{\check{f}(\by_t)}{2h^2}\|\by_t-\by_{t+1}\|^2.
\end{align}
Proposition~\ref{prop:fw} shows that there exists a constant $C_1>0$ 
such that $\check{f}(\by_t)\ge\frac{C_1}{n h^d}$.
Consequently, one has
\begin{align}
\label{eq:select-a}
	f(\by_{t+1})-f(\by_t)\ge\bar{a}\|\by_{t+1}-\by_t\|^2
	\text{ with }\bar{a}=\frac{C_1}{2n h^{d+2}}.
\end{align}
\end{proof}

\begin{lemma}
\label{lem:B}
Assume Assumptions~\ref{asm:RS} and \ref{asm:QM}, and $f(\by_1)>0$.
\begin{enumerate}
\item[\hyt{d1}] 
Assume furthermore the former half of the assumption~\hyl{a1}
in Theorem~\ref{thm:MS-GCG}:
The KDE $f$ is differentiable on $\cl(\Conv(\{\by_t\}_{t\ge\tau}))$ 
with some $\tau\in\mathbb{N}$. 
Then there exists $\bar{b}>0$ such that
\begin{align}
\label{eq:lem-b}
	\|\by_{t+1}-\by_t\|\ge\bar{b}\|\nabla f(\by_t)\|
\end{align}
holds for any $t\ge\tau$.
\item[\hyt{d2}] 
Instead, assume further the former half of Assumption~\ref{asm:LCG}:
The kernel $K$ is differentiable. 
Then there exists $\bar{b}>0$ such that \eqref{eq:lem-b} holds for any $t\in\bbN$.
\end{enumerate}
\end{lemma}
\begin{proof}[{Proof of Lemma~\ref{lem:B}}]
Proposition~\ref{prop:fw} ensures that $\check{f}(\by_t)>0$ for any $t\in\bbN$. 
Under the differentiability of the KDE $f$ at $\by_t$ with $t\ge\tau$,
the ordinary update rule of the MS algorithm \eqref{eq:MS-Iter} 
can be seen as a gradient ascent method with an adaptive step size:
\begin{align}
\label{eq:ms-grad}
	\by_{t+1}
	=\frac{\sum_{i=1}^n\check{K}(\|\frac{\by_t-\bx_i}{h}\|^2/2)\bx_i}
	{\sum_{i=1}^n\check{K}(\|\frac{\by_t-\bx_i}{h}\|^2/2)}
	=\by_t+\frac{h^2\cdot\{-\frac{1}{n h^{d+2}}
	\sum_{i=1}^n\check{K}(\|\frac{\by_t-\bx_i}{h}\|^2/2)(\by_t-\bx_i)\}}
	{\frac{1}{n h^d}\sum_{i=1}^n\check{K}(\|\frac{\by_t-\bx_i}{h}\|^2/2)}
	=\by_t+\frac{h^2}{\check{f}(\by_t)}\nabla f(\by_t).
\end{align}
The boundedness of $\check{K}$ (Lemma~\ref{lem:checkK})
implies that there exists a constant $C_2>0$ such that 
$|\check{K}(\|\frac{\by_t-\bx_i}{h}\|^2/2)|\le C_2$ for any $i=1,\ldots,n$
and any $t\in\bbN$, 
and hence $|\check{f}(\by_t)|\le\frac{C_2}{h^d}$.
Thus, for any $t\ge\tau$ one has 
\begin{align}
\label{eq:select-b}
	\|\by_{t+1}-\by_t\|
	=\frac{h^2}{|\check{f}(\by_t)|}\|\nabla f(\by_t)\|
	\ge\bar{b}\|\nabla f(\by_t)\|
	\text{ with }\bar{b}=\frac{h^{d+2}}{C_2},
\end{align}
proving the claim \hyl{d1}.

The claim \hyl{d2} follows from the differentiability of $f$ at every $\by_t$. 
\end{proof}

\begin{lemma}[Relative error condition]
\label{lem:C}
Assume Assumptions~\ref{asm:RS} and \ref{asm:QM}, and $f(\by_1)>0$.
\begin{enumerate}
\item[\hyt{e1}] 
Assume furthermore the assumption~\hyl{a1} in Theorem~\ref{thm:MS-GCG}:
The KDE $f$ is differentiable and has a Lipschitz-continuous gradient 
on $\cl(\Conv(\{\by_t\}_{t\ge\tau}))$ with some $\tau\in\mathbb{N}$. 
Then there exists $\bar{c}>0$ such that
\begin{align}
\label{eq:lem-c}
	\|\by_{t+1}-\by_t\|\ge\bar{c}\|\nabla f(\by_{t+1})\|
\end{align}
holds for any $t\ge\tau$.
\item[\hyt{e2}] 
Instead, assume furthermore 
Assumption~\ref{asm:LCG}:
The kernel $K$ is differentiable and has a Lipschitz-continuous gradient. 
Then there exists $\bar{c}>0$ such that \eqref{eq:lem-c} holds for any $t\in\bbN$.
\end{enumerate}
\end{lemma}
\begin{proof}[{Proof of Lemma~\ref{lem:C}}]
With the Lipschitz constant $L\ge0$ of $\nabla f$, 
one can find the relation
\begin{align}
\label{eq:select-c}
	\begin{split}
	\|\nabla f(\by_{t+1})\|
	&\le\|\nabla f(\by_t)\|+\|\nabla f(\by_{t+1})-\nabla f(\by_t)\|
	\quad(\because\text{Triangle inequality})\\
	&\le\frac{1}{\bar{b}}\|\by_{t+1}-\by_t\|+L\|\by_{t+1}-\by_t\|
	\quad(\because\text{Lemma~\ref{lem:B} and Lipschitz continuity of $\nabla f$})\\
	&=\frac{1}{\bar{c}}\|\by_{t+1}-\by_t\|
	\text{ with }\bar{c}=\biggl(\frac{1}{\bar{b}}+L\biggr)^{-1}.
	\end{split}
\end{align}
When $\nabla K$ is Lipschitz-continuous with a Lipschitz constant $C_3\ge0$,
one can set $L=\frac{C_3}{h^{d+2}}$ and $\bar{c}=\frac{h^{d+2}}{C_2+C_3}$
with a constant $C_2>0$ that bounds $|\check{K}(\|\frac{\by_t-\bx_i}{h}\|^2/2)|$ 
from above for every $i\in[n]$ and $t\in\bbN$.
\end{proof}

\textbf{Preliminaries for Proof of Theorems \ref{thm:MS-GCG} and \ref{thm:MS-Rate}}\\
Let $\delta\in(0,\infty]$, and let $\varphi:[0,\delta)\to[0,\infty)$ 
be a continuous concave function such that $\varphi(0)=0$ 
and $\varphi$ is continuously differentiable on $(0,\delta)$ 
with $\varphi'(u)>0$.
The concavity of $\varphi$ implies that $\varphi'$ is non-increasing
on $(0,\delta)$.
The {\L}ojasiewicz inequality \eqref{eq:Lojasiewicz-ineq} 
holds trivially with $(\bx',\bx)$ satisfying $g(\bx')-g(\bx)=0$.
Also, it is known that the {\L}ojasiewicz inequality \eqref{eq:Lojasiewicz-ineq} with $(\bx',\bx)$ 
such that $\bx\in\bar{U}(\bx',g,T,\epsilon,\delta)\coloneq\{\bx\in T\mid \|\bx'-\bx\|<\epsilon, g(\bx')-g(\bx)\in(0,\delta)\}$
is a special case of
\begin{align}
\label{eq:Lojasiewicz-ineq2}
	\varphi'(g(\bx')-g(\bx))\|\nabla g(\bx)\|\ge 1
	\text{ at }(\bx',\bx)\text{ such that }\bx\in\bar{U}(\bx',g,T,\epsilon,\delta)
\end{align}
with $\varphi(u)=\frac{u^{1-\theta}}{c(1-\theta)}$ where $c$ is a positive constant.
(One technical subtlety with this extended definition is that
we have excluded those $\bx$ with $g(\bx)=g(\bx')$ from
$\bar{U}(\bx',g,T,\epsilon,\delta)$, as those points would
make the left-hand side of~\eqref{eq:Lojasiewicz-ineq2} indeterminate.)
Note that
\cite{attouch2013convergence,frankel2015splitting}
call the function $\varphi$ a desingularizing function
because of its role in \eqref{eq:Lojasiewicz-ineq2}, where $\varphi\circ g$ is
in a sense resolving criticality of $g$ at $\bx'$.
Note also that for the choice $\varphi(u)=\frac{u^{1-\theta}}{c(1-\theta)}$, 
one has $\varphi'(u)=\frac{u^{-\theta}}{c}$,
recovering the original definition (Definition~\ref{def:Loj}) of
the {\L}ojasiewicz property. 
%
The following proof of the convergence of the mode estimate sequence $(\by_t)_{t\in\bbN}$ (Theorem \ref{thm:MS-GCG})
is not restricted to the specific choice $\varphi(u)=\frac{u^{1-\theta}}{c(1-\theta)}$ but 
holds with the general form \eqref{eq:Lojasiewicz-ineq2} of the {\L}ojasiewicz inequality.
The specific choice $\varphi(u)=\frac{u^{1-\theta}}{c(1-\theta)}$,
on the other hand, will help derive the worst-case bound of the convergence rate
in Theorem~\ref{thm:MS-Rate}.

\textbf{Proof of Theorem \ref{thm:MS-GCG}}\\
We here provide a proof of Theorem~\ref{thm:MS-GCG}
on the ground of \cite[Theorem 3.2]{attouch2013convergence} 
and \cite[Theorem 3.1]{frankel2015splitting}.

\begin{proof}[Proof of Theorem~\ref{thm:MS-GCG}]
The density estimate sequence $(f(\by_t))_{t\in\bbN}$ converges 
under Assumptions~\ref{asm:RS} and \ref{asm:QM} since it is 
a bounded non-decreasing sequence (Proposition~\ref{prop:dens}).
Also, as $f(\by_1)>0$, for every $t\ge2$ $\by_t$ lies in the convex hull
$\Conv(\{\bx_i\}_{i\in[n]})$ of data points, which is a compact set. 
Thus, there exist an accumulation point $\tilde{\by}\in\Conv(\{\bx_i\}_{i\in[n]})$
of the mode estimate sequence $(\by_t)_{t\in\bbN}$ and a subsequence 
$(\by_{t'})_{t'\in N}$ of $(\by_t)_{t\in\bbN}$ (with $N\subseteq\bbN$)
that converges to the accumulation point $\tilde{\by}$ as $t'\to\infty$.
Also, $\tilde{\by}\in\cl(\Conv(\{\by_t\}_{t\ge\tau}))$ obviously holds for any $\tau\in\bbN$.
When there exists $t'\in N$ such that $f(\tilde{\by})=f(\by_{t'})$, 
Lemma \ref{lem:A} obviously shows the convergence of $(\by_t)_{t\in\bbN}$ to $\tilde{\by}$:
Assume $\by_{t'+1}\not=\by_{t'}$.
One then has $f(\by_{t'+1})\ge f(\by_{t'})+\bar{a}\|\by_{t'+1}-\by_{t'}\|^2>f(\tilde{\by})$
since $f(\by_{t'})=f(\tilde{\by})$ and $\bar{a}>0$.
It then follows from the monotonicity of $(f(\by_t))_{t\in\mathbb{N}}$
that $f(\tilde{\by})=\lim_{t'\in N, t\to\infty}f(\by_{t'})>f(\tilde{\by})$,
which is a contradiction.
On the other hand, if $\by_{t'+1}=\by_{t'}$, then
one has $\by_t=\by_{t'}$ for any $t\ge t'$ and hence $\tilde{\by}=\by_{t'}$. 
%
We therefore consider in what follows the remaining case where $f(\tilde{\by})>f(\by_t)$ for all $t\in\mathbb{N}$. 
The assumption~\hyl{a2} ensures 
that there exists a positive constant $\epsilon$ 
such that 
the KDE $f$ satisfies the {\L}ojasiewicz inequality \eqref{eq:Lojasiewicz-ineq} 
at least with any $(\bx',\bx)=(\tilde{\by},\by)$ 
such that $\by\in U(\tilde{\by},f,\cl(\Conv(\{\by_s\}_{s\ge\tau})),\epsilon)$
for some integer $\tau$. 

As we want to use the general form~\eqref{eq:Lojasiewicz-ineq2}
of the Lojasiewicz inequality, we have to further restrict
the region where the {\L}ojasiewicz inequality to hold
from $U(\tilde{\by},f,\cl(\Conv(\{\by_s\}_{s\ge\tau})),\epsilon)$
to $\bar{U}(\tilde{\by},f,\cl(\Conv(\{\by_s\}_{s\ge\tau})),\epsilon,\delta)$
in order to ensure that
$f(\tilde{\by})-f(\by)$ is in the domain $[0,\delta)$
of the desingularizing function $\varphi$. 
Denoting $r_t\coloneq f(\tilde{\by})-f(\by_t)>0$, the convergence of 
the density estimate sequence $(f(\by_t))_{t\in\bbN}$ and the definition of $\tilde{\by}$
imply that the sequence $(r_t)_{t\in\bbN}$ is positive, non-increasing,
and converging to 0 as $t\to\infty$. 
The facts, $\by_{t'}\to\tilde{\by}$ and $r_t\to0$, as well as the continuity of $\varphi$, 
imply the existence of a finite integer $\tau'\ge\tau$ in $N$ 
such that $r_t\in[0,\delta)$ holds for any $t\ge\tau'$, and
that the inequality 
\begin{align}
\label{eq:At4}
	\|\tilde{\by}-\by_{\tau'}\|+2\sqrt{\frac{r_{\tau'}}{\bar{a}}}+\frac{1}{\bar{a}\bar{c}}\varphi(r_{\tau'})<\epsilon
\end{align}
holds.
It should be noted that if the assumptions~\hyl{a1} and \hyl{a2}
hold with some $\tau\in\mathbb{N}$, they also hold with the above $\tau'$
since $\{\by_s\}_{s\ge\tau'}\subseteq\{\by_s\}_{s\ge\tau}$ with $\tau'\ge\tau$. 
Using the {\L}ojasiewicz property of the KDE $f$ on $\bar{U}(\tilde{\by},f,\cl(\Conv(\{\by_s\}_{s\ge\tau'})),\epsilon,\delta)$,
the inequality~\eqref{eq:At4}, and assumption \hyl{a1},
we prove below that the mode estimate sequence $(\by_t)_{t\in\bbN}$ does not endlessly wander 
and does converge to $\tilde{\by}$, and that $\tilde{\by}$ is a critical point of the KDE $f$.

\textbf{Two key claims:}
We will establish the following two claims 
for any $t\ge\tau'+1$, which are the key to proving Theorem~\ref{thm:MS-GCG}. 
\begin{claim}
\label{claim:At6}
$\by_t$ satisfies
\begin{align}
\label{eq:At6}
\by_t\in\bar{U}(\tilde{\by},f,\cl(\Conv(\{\by_s\}_{s\ge\tau'})),\epsilon,\delta).
\end{align}
In other words, the {\L}ojasiewicz inequality~\eqref{eq:Lojasiewicz-ineq2} with $(\bx',\bx)=(\tilde{\by},\by_t)$ holds. 
\end{claim}
\begin{claim}
\label{claim:At7}
$\{\by_s\}_{s\in\{\tau',\ldots,t+1\}}$ satisfies 
\begin{align}
\label{eq:At7}
	\sum_{s=\tau'+1}^t\|\by_{s+1}-\by_s\|
	+\|\by_{t+1}-\by_t\|
	\le\|\by_{\tau'+1}-\by_{\tau'}\|
	+\frac{1}{\bar{a}\bar{c}}\{\varphi(r_{\tau'+1})-\varphi(r_{t+1})\}.
\end{align}
\end{claim}

\textbf{Auxiliary results:}
We here provide two auxiliary results to be used in the succeeding proof.
First, one has
\begin{align}
\label{eq:At8}
	\begin{split}
	\|\by_{\tau'+1}-\by_{\tau'}\|
	&\le\sqrt{\frac{r_{\tau'}-r_{\tau'+1}}{\bar{a}}}
	\quad(\because\text{Lemma~\ref{lem:A}})\\
	&\le\sqrt{\frac{r_{\tau'}}{\bar{a}}}
	\quad(\because r_{\tau'+1}\ge0).
	\end{split}
\end{align}
Secondly, we show the following auxiliary lemma,
which will be used in proving~\eqref{eq:At7}
from~\eqref{eq:At6} via making use of the {\L}ojasiewicz property. 
\begin{lemma}
\label{lem:lem}
If $\by_t$ with $t\ge\tau$ satisfies Claim~\ref{claim:At6}, that is, if $\by_t\in\bar{U}(\tilde{\by},f,\cl(\Conv(\{\by_s\}_{s\ge\tau'})),\epsilon,\delta)$ holds, then
\begin{align}
\label{eq:At9}
	2\|\by_{t+1}-\by_t\|
	\le\|\by_t-\by_{t-1}\|+\frac{1}{\bar{a}\bar{c}}\{\varphi(r_t)-\varphi(r_{t+1})\}.
\end{align}
\end{lemma}
\begin{proof}[Proof of Lemma~\ref{lem:lem}]
Since \eqref{eq:At9} holds trivially if $\by_t=\by_{t-1}$,
we consider the case $\by_t\neq\by_{t-1}$.
When $\by_t\in\bar{U}(\tilde{\by},f,\cl(\Conv(\{\by_s\}_{s\ge\tau'})),\epsilon,\delta)$,
the {\L}ojasiewicz inequality~\eqref{eq:Lojasiewicz-ineq2}
with $(\bx',\bx)=(\tilde{\by},\by_t)$ holds. 
Noting that $0<r_{t+1}\le r_t<\delta$ holds, one has 
\begin{align}
\label{eq:Direct}
	\begin{split}
	\varphi(r_t)-\varphi(r_{t+1})
	&=\int_{r_{t+1}}^{r_t}\varphi'(u)\,du\\ 
	&\ge\varphi'(r_t)(r_t-r_{t+1})
	\quad(\because\text{$\varphi'$ is positive and non-increasing})\\
	&\ge\varphi'(r_t)\bar{a}\|\by_{t+1}-\by_t\|^2
	\quad(\because\text{Lemma~\ref{lem:A}})\\
	&\ge\frac{1}{\|\nabla f(\by_t)\|}\bar{a}\|\by_{t+1}-\by_t\|^2
	\quad(\because\mbox{{\L}ojasiewicz inequality~\eqref{eq:Lojasiewicz-ineq2} with $(\bx',\bx)=(\tilde{\by},\by_t)$})\\
	&\ge\bar{a}\bar{c}\frac{\|\by_{t+1}-\by_t\|^2}{\|\by_t-\by_{t-1}\|}
	\quad(\because\text{Lemma~\ref{lem:C}}).
	\end{split}
\end{align}
The inequality $2\sqrt{\alpha\beta}\le\alpha+\beta$ for $\alpha,\beta\ge0$ yields
\begin{align}
	\begin{split}
	2\|\by_{t+1}-\by_t\|
	&=2\sqrt{\|\by_{t+1}-\by_t\|^2}\\
	&\le2\sqrt{\|\by_t-\by_{t-1}\|\frac{1}{\bar{a}\bar{c}}\{\varphi(r_t)-\varphi(r_{t+1})\}}
	\quad(\because\text{\eqref{eq:Direct}})\\
	&\le \|\by_t-\by_{t-1}\|+\frac{1}{\bar{a}\bar{c}}\{\varphi(r_t)-\varphi(r_{t+1})\}.
	\end{split}
\end{align}
This concludes the proof of Lemma~\ref{lem:lem}.
\end{proof}

\textbf{Proof that Claims~\ref{claim:At6} and \ref{claim:At7} hold for $t=\tau'+1$:}
Here we prove Claims~\ref{claim:At6} and \ref{claim:At7} for $t=\tau'+1$.
One has
\begin{align}
	\begin{split}
	\|\tilde{\by}-\by_{\tau'+1}\|
	&\le\|\tilde{\by}-\by_{\tau'}\|
	+\|\by_{\tau'+1}-\by_{\tau'}\|
	\quad(\because\text{Triangle inequality})\\
	&\le\|\tilde{\by}-\by_{\tau'}\|
	+\sqrt{\frac{r_{\tau'}}{\bar{a}}}
	\quad(\because\text{\eqref{eq:At8}})\\
	&<\epsilon
	\quad(\because\text{\eqref{eq:At4}}),
	\end{split}
\end{align}
which, together with $0<r_{\tau'+1}\le r_{\tau'}<\delta$,
implies \eqref{eq:At6} with $t=\tau'+1$,
proving Claim~\ref{claim:At6} for $t=\tau'+1$. 
Also, Claim~\ref{claim:At6} with $t=\tau'+1$ implies, via Lemma~\ref{lem:lem},
the inequality \eqref{eq:At9} with $t=\tau'+1$, which reads 
\begin{align}
	2\|\by_{\tau'+2}-\by_{\tau'+1}\|
	\le\|\by_{\tau'+1}-\by_{\tau'}\|
	+\frac{1}{\bar{a}\bar{c}}\{\varphi(r_{\tau'+1})-\varphi(r_{\tau'+2})\},
\end{align}
which is nothing other than~\eqref{eq:At7} with $t=\tau'+1$,
thereby proving Claim~\ref{claim:At7} for $t=\tau'+1$. 

\textbf{Proof that Claims~\ref{claim:At6} and \ref{claim:At7} hold for $t\ge\tau'+1$:}
Now that we have seen that Claim~\ref{claim:At7} holds for $t=\tau'+1$,
we next prove Claim~\ref{claim:At7} to hold for every $t\ge\tau'+1$ by induction. 
For this purpose, we prove Claims~\ref{claim:At6} and \ref{claim:At7} for $t=u+1$ under
the assumption that Claims~\ref{claim:At6} and \ref{claim:At7} hold for $t=u\ge\tau'+1$.
One has 
\begin{align}
\label{eq:Atabove}
	\begin{split}
	\|\tilde{\by}-\by_{u+1}\|
	&\le\|\tilde{\by}-\by_{\tau'}\|
	+\|\by_{\tau'+1}-\by_{\tau'}\|
	+\sum_{s=\tau'+1}^u\|\by_{s+1}-\by_s\|
	\quad(\because\text{Triangle inequality})\\
	&\le\|\tilde{\by}-\by_{\tau'}\|
	+2\|\by_{\tau'+1}-\by_{\tau'}\|
	+\frac{1}{\bar{a}\bar{c}}\{\varphi(r_{\tau'+1})-\varphi(r_{u+1})\}
	-\|\by_{u+1}-\by_u\|
	\quad(\because\eqref{eq:At7}\text{ with }t=u)\\
	&\le\|\tilde{\by}-\by_{\tau'}\|
	+2\|\by_{\tau'+1}-\by_{\tau'}\|
	+\frac{1}{\bar{a}\bar{c}}\varphi(r_{\tau'+1})
	\quad(\because\|\by_{u+1}-\by_u\|\ge0\text{ and }\varphi(r_{u+1})\ge0)\\
	&\le\|\tilde{\by}-\by_{\tau'}\|
	+2\sqrt{\frac{r_{\tau'}}{\bar{a}}}
	+\frac{1}{\bar{a}\bar{c}}\varphi(r_{\tau'})
	\quad(\because\text{\eqref{eq:At8} and }\varphi(r_{\tau'})\ge\varphi(r_{\tau'+1}))\\
	&<\epsilon
	\quad(\because\text{\eqref{eq:At4}}),
	\end{split}
\end{align}
which, together with Claim~\ref{claim:At6} for $t=u$ and $0<r_{u+1}\le r_u<\delta$,
implies Claim~\ref{claim:At6} to hold for $t=u+1$.
Also, this result ensures, via Lemma~\ref{lem:lem},
that \eqref{eq:At9} holds with $t=u+1$.
Adding \eqref{eq:At9} with $t=u+1$ to \eqref{eq:At7} with $t=u$
then shows that \eqref{eq:At7} holds with $t=u+1$,
proving Claim~\ref{claim:At7} to hold for $t=u+1$.
As Claims~\ref{claim:At6} and \ref{claim:At7} have been shown to hold for $t=\tau'+1$,
the above argument proves, by induction, that
Claim~\ref{claim:At7} holds for every $t\ge\tau'+1$.

\textbf{Claim~\ref{claim:At7} for every $t\ge\tau'+1$ implies convergence:}
From~\eqref{eq:At7}, one has for any $t\ge\tau'+1$ 
\begin{align}
	\begin{split}
	\sum_{s=\tau+1}^t\|\by_{s+1}-\by_s\|
	&\le\|\by_{\tau'+1}-\by_{\tau'}\|
	+\frac{1}{\bar{a}\bar{c}}\{\varphi(r_{\tau'+1})-\varphi(r_{t+1})\}
	-\|\by_{t+1}-\by_t\|\\
	&\le\|\by_{\tau'+1}-\by_{\tau'}\|
	+\frac{1}{\bar{a}\bar{c}}\varphi(r_{\tau'+1})
	\quad(\because\|\by_{t+1}-\by_t\|\ge0\text{ and }\varphi(r_{t+1})\ge0).
	\end{split}
\end{align}
Taking the limit $t\to\infty$ yields
\begin{align}
  \sum_{s=\tau'+1}^\infty\|\by_{s+1}-\by_s\|
  \le\|\by_{\tau'+1}-\by_{\tau'}\|
	+\frac{1}{\bar{a}\bar{c}}\varphi(r_{\tau'+1}),
\end{align}
which implies
\begin{align}
	\begin{split}
	\sum_{s=1}^\infty\|\by_{s+1}-\by_s\|
	&=\sum_{s=1}^{\tau'}\|\by_{s+1}-\by_s\|
	+\sum_{s=\tau'+1}^\infty\|\by_{s+1}-\by_s\|\\
	&=\sum_{s=1}^{\tau'}\|\by_{s+1}-\by_s\|
	+\|\by_{\tau'+1}-\by_{\tau'}\|
	+\frac{1}{\bar{a}\bar{c}}\varphi(r_{\tau'+1})
	<\infty.
	\end{split}
\end{align}
This shows that the trajectory of $(\by_t)_{t\in\bbN}$ is of finite length,
which in turn implies 
that $(\by_t)_{t\in\bbN}$ converges. 
As the limit $\lim_{t\to\infty}\by_t$ is a unique accumulation point
of $(\by_t)_{t\in\bbN}$, it must be $\tilde{\by}$
since $\by_{t'}\to\tilde{\by}$.
Additionally, from Lemma~\ref{lem:B}, one has 
\begin{align}
	\sum_{s=\tau'}^\infty\|\nabla f(\by_s)\|
	\le\frac{1}{\bar{b}}\sum_{s=\tau'}^\infty\|\by_{s+1}-\by_s\|<\infty,
\end{align}
which implies $\lim_{t\to\infty}\|\nabla f(\by_t)\|=0$.
Since the gradient of the KDE $f$ is Lipschitz-continuous 
with a Lipschitz constant $L\ge0$ 
on $\cl(\Conv(\{\by_y\}_{t\ge\tau'}))$ 
due to the assumption~\hyl{a1}, one has that
\begin{align}
	\|\nabla f(\tilde{\by})\|
	\le\lim_{t\to\infty}\{\|\nabla f(\by_t)\|+\|\nabla f(\tilde{\by})-\nabla f(\by_t)\|\}	
	\le\lim_{t\to\infty}\{\|\nabla f(\by_t)\|+L\|\tilde{\by}-\by_t\|\}
	=0,
\end{align}
which implies that the limit $\tilde{\by}=\lim_{t\to\infty}\by_t$
is a critical point of $f$. 
\end{proof}

\textbf{Proof of Theorem \ref{thm:MS-Rate}}\\
For a desingularizing function $\varphi(u)$, 
define $\Phi(u)$ to be a primitive function (indefinite integral)
of $-(\varphi')^2$.
For the specific choice of the desingularizing function
$\varphi(u)=\frac{u^{1-\theta}}{c(1-\theta)}$,
one has
\begin{align}
\label{eq:RES0}
	\varphi'(u)
	=\frac{u^{-\theta}}{c},\quad
	\Phi(u)
	=\begin{cases}
	-\frac{u^{1-2\theta}}{c^2(1-2\theta)}&\text{if }\theta\in[0,\frac{1}{2}),\\
	-\frac{\log(u)}{c^2}&\text{if }\theta=\frac{1}{2},\\
	-\frac{u^{1-2\theta}}{c^2(1-2\theta)}&\text{if }\theta\in(\frac{1}{2},1),
	\end{cases}\quad
	\Phi^{-1}(u)
	=\begin{cases}
	\exp(-c^2 u)&\text{if }\theta=\frac{1}{2},\\
	\{c^2(2\theta-1)u\}^{-\frac{1}{2\theta-1}}&\text{if }\theta\in(\frac{1}{2},1).
	\end{cases}
\end{align}
These functional forms will be used in proving Theorem~\ref{thm:MS-Rate}.

Now, we provide a proof of Theorem~\ref{thm:MS-Rate}, 
which is based on the proof of \cite[Theorem 3.5]{frankel2015splitting}.

\begin{proof}[Proof of Theorem~\ref{thm:MS-Rate}]
In the proof of Theorem~\ref{thm:MS-GCG} we have established the following facts: 
If there exists $t'\in\bbN$ such that $f(\bar{\by})=f(\by_{t'})$
then $\by_t=\bar{\by}$ for any $t\ge t'$, that is,
$(\by_t)_{t\in\bbN}$ converges in a finite number of iterations.
If otherwise, then 
there exists $\tau\in\bbN$ such that Claim~\ref{claim:At6} holds
for any $t\ge\tau$, that is,
$\by_t\in\bar{U}(\bar{\by},f,\cl(\Conv(\{\by_s\}_{s\ge\tau})),\epsilon,\delta)$ 
holds for any $t\ge\tau$,
or equivalently, the {\L}ojasiewicz inequality \eqref{eq:Lojasiewicz-ineq2}
with $(\bx',\bx)=(\bar{\by},\by_t)$ holds for any $t\ge\tau$. 
If $\by_t$ with $t\ge\tau$ satisfies Claim~\ref{claim:At6},
then one has
\begin{align}
\label{eq:INEQ2}
\begin{split}
	\Phi(r_{t+1})-\Phi(r_t)
	&=\int_{r_{t+1}}^{r_t}\{\varphi'(u)\}^2\,du
	\quad(\because\text{Definition of $\Phi$})\\
	&\ge\{\varphi'(r_t)\}^2 (r_t-r_{t+1})
	\quad(\because\text{$\varphi'$ is positive and non-increasing})\\
	&\ge\{\varphi'(r_t)\}^2 \bar{a}\|\by_{t+1}-\by_t\|^2
	\quad(\because\text{Lemma~\ref{lem:A}})\\
	&\ge\{\varphi'(r_t)\}^2 \bar{a} \{\bar{b}\|\nabla f(\by_t)\|\}^2
	\quad(\because\text{Lemma~\ref{lem:B}})\\
	&\ge\bar{a}\bar{b}^2
	\quad(\because\text{{\L}ojasiewicz inequality~\eqref{eq:Lojasiewicz-ineq2}
	with $(\bx',\bx)=(\bar{\by},\by_t)$}).
\end{split}
\end{align}
Now Claim~\ref{claim:At6} holds for any $t\ge \tau$, 
which implies
\begin{align}
\label{eq:INEQ3}
	\Phi(r_t)-\Phi(r_\tau)
	=\sum_{s=\tau}^{t-1}\{\Phi(r_{s+1})-\Phi(r_s)\}
	\ge\bar{a}\bar{b}^2(t-\tau-2).
\end{align}

We discuss the two cases $\theta\in[0,\frac{1}{2})$
and $\theta\in[\frac{1}{2},1)$ separately.

\textbf{Case $\theta\in[0,\frac{1}{2})$:}
We claim that in this case the algorithm converges in a finite number
of iterations.
If otherwise, the inequality~\eqref{eq:INEQ3} should hold for any $t\ge\tau$.
When we take the limit $t\to\infty$, 
the right-hand side of~\eqref{eq:INEQ3} goes to infinity, 
which contradicts the fact that the left-hand side remains finite, 
by noting that one has $\lim_{u\to0}\Phi(u)=0$
with $\Phi(u)=-\frac{u^{1-2\theta}}{c^2(1-2\theta)}$
and that $r_t\to0$ as $t\to\infty$.
This contradiction implies the finite-time convergence of $(\by_t)_{t\in\bbN}$.

\textbf{Case $\theta\in[\frac{1}{2},1)$:}
We may suppose that $r_t>0$ holds for any $t\in\bbN$,
and so \eqref{eq:INEQ3} holds for any $t\ge\tau$.
Recalling the functional form of $\Phi(u)$ as in~\eqref{eq:RES0}, 
one has $\lim_{t\to\infty}\Phi(r_t)=\infty$
with $\theta\in[\frac{1}{2},1)$.
Assume $\Phi(r_\tau)\ge0$. (If it is not the case
one can always redefine $\tau$ to a larger value with which
$\Phi(r_\tau)\ge0$ is satisfied.) 
One then has $\Phi(r_t)\ge\bar{a}\bar{b}^2(t-\tau-2)+\Phi(r_\tau)\ge\bar{a}\bar{b}^2(t-\tau-2)$,
which allows us to 
obtain the convergence rate evaluation for $r_t=f(\bar{\by})-f(\by_t)$, 
namely,
\begin{align}
\label{eq:RES1}
	r_t\le\Phi^{-1}(\bar{a}\bar{b}^2(t-\tau-2)).
\end{align}
With the explicit form of $\Phi^{-1}$ given in~\eqref{eq:RES0}, 
one has 
\begin{align}
	r_t\le
	\begin{cases}
	\exp(-c^2\bar{a}\bar{b}^2(t-\tau-2))
	=O(q^{2t})
	&\text{if }\theta=\frac{1}{2},\\
	\{c^2(2\theta-1)\bar{a}\bar{b}^2(t-\tau-2)\}^{-\frac{1}{2\theta-1}}
	=O(t^{-\frac{1}{2\theta-1}})
	&\text{if }\theta\in(\frac{1}{2},1),
	\end{cases}
\end{align}
where $q=\exp(-c^2\bar{a}\bar{b}^2/2)\in(0,1)$.
For the convergence rate evaluation for $\|\by_t-\bar{\by}\|$,
we have 
\begin{align}
	\begin{split}
	\varphi(r_t)-\varphi(r_{t+1})
	&=\int_{r_{t+1}}^{r_t}\varphi'(u)\,du\\ 
	&\ge\varphi'(r_t)(r_t-r_{t+1})
	\quad(\because\text{$\varphi'$ is positive and non-increasing})\\
	&\ge\varphi'(r_t)\bar{a}\|\by_{t+1}-\by_t\|^2
	\quad(\because\text{Lemma~\ref{lem:A}})\\
	&\ge\varphi'(r_t)\bar{a}\|\by_{t+1}-\by_t\|\bar{b}\|\nabla f(\by_t)\|
	\quad(\because\text{Lemma~\ref{lem:B}})\\
	&\ge\bar{a}\bar{b}\|\by_{t+1}-\by_t\|
	\quad(\because\text{{\L}ojasiewicz inequality~\eqref{eq:Lojasiewicz-ineq2}
	with $(\bx',\bx)=(\bar{\by},\by_t)$}),
	\end{split}
\end{align}
which in turn yields
\begin{align}
\label{eq:RES2}
	\|\by_t-\bar{\by}\|
	\le\sum_{s=t}^\infty \|\by_{s+1}-\by_s\|
	\le\frac{1}{\bar{a}\bar{b}}\sum_{s=t}^\infty\{\varphi(r_s)-\varphi(r_{s+1})\}
	\le\frac{1}{\bar{a}\bar{b}}\varphi(r_t)
	\le\frac{1}{\bar{a}\bar{b}}\varphi(\Phi^{-1}(\bar{a}\bar{b}^2(t-\tau-2)))
\end{align}
from \eqref{eq:RES1}.
According to the calculation \eqref{eq:RES0} for $\varphi(u)=\frac{u^{1-\theta}}{c(1-\theta)}$,
one can obtain the exponential-rate convergence when $\theta=\frac{1}{2}$
and polynomial-rate convergence when $\theta\in(\frac{1}{2},1)$:
\begin{align}
	\|\by_t-\bar{\by}\|
	\le\frac{\{\Phi^{-1}(\bar{a}\bar{b}^2(t-\tau-2))\}^{1-\theta}}{\bar{a}\bar{b}c(1-\theta)}
	=\begin{cases}
	\frac{\{\exp(-c^2\bar{a}\bar{b}^2(t-\tau-2))\}^{\frac{1}{2}}}{\bar{a}\bar{b}c(1-\theta)}
	=O(q^t)
	&\text{if }\theta=\frac{1}{2},\\
	\frac{\{\{c^2(2\theta-1)\bar{a}\bar{b}^2(t-\tau-2)\}^{-\frac{1}{2\theta-1}}\}^{1-\theta}}{\bar{a}\bar{b}c(1-\theta)}
	=O(t^{-\frac{1-\theta}{2\theta-1}})
	&\text{if }\theta\in(\frac{1}{2},1).
	\end{cases}
\end{align}
This concludes the proof for all the cases, \hyl{b1}, \hyl{b2}, and \hyl{b3}.
\end{proof}

\textbf{Proof of Theorem~\ref{thm:MS-Upper}}\\
Theorem~\ref{thm:MS-Upper} is proved using 
\cite[Proposition 4.3]{d2005explicit} as follows:

\begin{proof}[{Proof of Theorem~\ref{thm:MS-Upper}}]
As the kernel $K$ is assumed to be piecewise polynomial, 
the KDE $f$ is also piecewise polynomial, that is,
there exists a finite collection $\{S_l\}_{l=1}^L$ of subdomains $S_l\subseteq\mathbb{R}^d$, $l\in[L]$
that forms a partition of the entire domain $\mathbb{R}^d$ of the KDE $f$
such that in each subdomain $S_l$
the restriction of the KDE $f$ to $S_l$ is the same as
the restriction of the polynomial $f_l$ to $S_l$.

\textbf{Case I:}
When the critical point $\bar{\by}$ of the KDE $f$ lies
in the interior of one of the subdomains, say $S_l$, 
then one can take $\epsilon>0$ small enough so that 
$U(\bar{\by},f,\mathbb{R}^d,\epsilon)$ is contained in the subdomain $S_l$.
Then the KDE $f$ is equal to the polynomial $f_l$ in $U(\bar{\by},f,\mathbb{R}^d,\epsilon)$.
The polynomial $f_l$ is not constant by assumption,
and its degree $k$ is at least 2 as $\bar{\by}$ is a critical point of $f_l$. 
Therefore, any upper bound of the {\L}ojasiewicz exponent
of that polynomial is an upper bound of the {\L}ojasiewicz exponent
of the KDE $f$.

\textbf{Case II:}
We next assume in the following that $\bar{\by}$ is located on
a boundary of several subdomains.
Let $S_1,\ldots,S_{L'}$ (with $2\le L'\le L$) be the subdomains each of which
has a non-empty intersection with the $\epsilon$-neighbor of $\bar{\by}$
for any $\epsilon>0$.
One has $f_1(\bar{\by})=\cdots=f_{L'}(\bar{\by})$.
Because of the assumption that the kernel $K$ is of class $C^1$,
one also has $\nabla f_l(\bar{\by})=\bm{0}$ for all $l\in[L']$.
For any $l\in[L']$, the polynomial $f_l$ is not constant by assumption,
and its degree $k_l$ is at least 2 as $\bar{\by}$ is a critical point of $f_l$. 
One can therefore assume that for any $l\in[L']$ $f_l$ has 
the {\L}ojasiewicz property, that is,
there exist $\epsilon_l>0$, $c_l>0$, and $\theta_l\in[0,1)$ 
such that for any $\by\in U(\bar{\by},f,S_l,\epsilon_l)$
$f_l$ satisfies the {\L}ojasiewicz inequality
\begin{align}
	\|\nabla f_l(\by)\|\ge c_l\{f_l(\bar{\by})-f_l(\by)\}^{\theta_l}.
\end{align}
We show that under these conditions $f$ has the {\L}ojasiewicz property
at $\bar{\by}$. 

Let $\epsilon_{\mathrm{min}}\coloneq\min_{l\in[L']}\epsilon_l$,
$\theta_{\mathrm{max}}\coloneq\max_{l\in[L']}\theta_l$, and
\begin{align}
	A\coloneq\max_{l\in[L']}\sup_{\by\in U(\bar{\by},f,S_l,\epsilon_{\mathrm{min}})}
	\{f_l(\bar{\by})-f_l(\by)\}>0.
\end{align}
Take any $\by\in U(\bar{\by},f,\mathbb{R}^d,\epsilon_{\mathrm{min}})$. 
Then there exists an index $l(\by)\in[L']$ such that $\by\in S_{l(\by)}$, and 
\begin{align}
	\begin{split}
	\|\nabla f(\by)\|
	&=\|\nabla f_{l(\by)}(\by)\|
\\
	&\ge c_{l(\by)}\{f_{l(\by)}(\bar{\by})-f_{l(\by)}(\by)\}^{\theta_{l(\by)}}
	\quad(\mbox{$\because$ {\L}ojasiewicz property of $f_l$ at $\bar{\by}$})
\\
	&=c_{l(\by)}A^{\theta_{l(\by)}}
	\left\{\frac{f_{l(\by)}(\bar{\by})-f_{l(\by)}(\by)}{A}\right\}^{\theta_{l(\by)}}
\\
	&\ge c_{l(\by)}A^{\theta_{l(\by)}}
	\left\{\frac{f_{l(\by)}(\bar{\by})-f_{l(\by)}(\by)}{A}\right\}^{\theta_{\mathrm{max}}}
\\
	&=c_{l(\by)}A^{\theta_{l(\by)}-\theta_{\mathrm{max}}}
	\{f_{l(\by)}(\bar{\by})-f_{l(\by)}(\by)\}^{\theta_{\mathrm{max}}}
\\
	&\ge c'
	\{f(\bar{\by})-f(\by)\}^{\theta_{\mathrm{max}}},
	\quad c'\coloneq\min_{l\in[L']}c_lA^{\theta_l-\theta_{\mathrm{max}}}>0,
	\end{split}
\end{align}
which shows that $f$ has the {\L}ojasiewicz property at $\bar{\by}$
with the exponent $\theta_{\mathrm{max}}$.
The arguments so far have proved that,
when the kernel $K$ is piecewise polynomial and of class $C^1$, 
the {\L}ojasiewicz exponent of the KDE $f$ at any critical point
is bounded from above by the largest {\L}ojasiewicz exponent
of the related polynomials $f_1,\ldots,f_{L'}$.

For a polynomial,
\cite[Proposition 4.3]{d2005explicit} gives 
an upper bound of the {\L}ojasiewicz exponent at its critical point,
and the bound is increasing in the degree of the polynomial.
When the kernel $K$ is piecewise polynomial with the maximum degree $k$,
the polynomials $\{f_l\}$ appearing as the restrictions of the KDE $f$
are of degrees at most $k$. 
Thus, by substituting the possible maximum degree $k$ of a piecewise polynomial KDE 
into the bound in \cite[Proposition 4.3]{d2005explicit},
one can obtain the upper bound 
of the {\L}ojasiewicz exponent of the KDE at its critical point
as in~\eqref{eq:Loja-exp-UB}, proving the theorem. 
\end{proof}

Note that we can 
obtain an alternative upper bound of the {\L}ojasiewicz exponent
at a critical point of a polynomial, 
which is valid when the critical point is a local maximum of the polynomial. 
It is given by $\theta\le1-1/\{(k-1)^d+1\}$
for a degree-$k$ polynomial of $d$ variables,
according to \cite{gwozdziewicz1999lojasiewicz},
and is better than the upper bound
$1-1/\max\{k(3k-4)^{k-1},2k(3k-3)^{k-2}\}$ used in the above proof,
the latter of which does not require
that the critical point is a local maximum.
Therefore, in the above proof, if $\bar{\by}$ is a local maximum
of the KDE $f$ and lies in the interior of one subdomain,
one has the better upper bound $\theta\le1-1/\{(k-1)^d+1\}$.
Although the better upper bound also applies to Case II if
$\bar{\by}$ is a local maximum of each of all the polynomials
$f_1,\ldots,f_{L'}$, in general it is not applicable, since a local maximum
of $f$ is not necessarily a local maximum of $f_l$.

\section*{S2 Supplement to Table \protect\ref{tab:Kernel}}
\label{sec:TabKer}
Here, we describe supplementary explanation to Table~\ref{tab:Kernel},
especially that of the fact that the kernels shown in Table~\ref{tab:Kernel}
satisfy Assumption~\ref{asm:LP}.
Throughout this section, we let $q(\bx,y)\coloneq 1-\|\bx-\bx_i\|^2$
and $r(\bx,y)\coloneq y$, both of which are polynomial and real analytic in $(\bx^\top,y)^\top$,
and omit the bandwidth $h$.

\textbf{Analytic Kernels}\\
It is clear that the Gaussian, logistic, and Cauchy kernels are real analytic functions.

\textbf{Piecewise Polynomial Kernels}\\
For a positive integer $p$,
let $K_p(\cdot-\bx_i)\coloneq C_p\{(q(\cdot,y))_+\}^p$ with a normalizing
coefficient $C_p>0$.
The Epanechnikov kernel $K(\cdot-\bx_i)=C_{\rm ep}(1-\|\cdot-\bx_i\|^2)_+$,
the biweight kernel $K(\cdot-\bx_i)=C_{\rm bw}\{(1-\|\cdot-\bx_i\|^2)_+\}^2$,
and the triweight kernel $K(\cdot-\bx_i)=C_{\rm tw}\{(1-\|\cdot-\bx_i\|^2)_+\}^3$
fall within this category with $p=1,2,3$, respectively.
The kernel $K_p(\cdot-\bx_i)$ is a semialgebraic function,
because its graph 
\begin{align}
	\begin{split}
	&\{(\bx,y)\in\bbR^{d+1}\mid y=C_p\{(q(\bx,y))_+\}^p\}\\
	&=\bigl(
	\{(\bx,y)\in\bbR^{d+1}\mid q(\bx,y)>0\}	
	\cap
	\{(\bx,y)\in\bbR^{d+1}\mid g_p(\bx,y)\coloneq (q(\bx,y))^p-y/C_p=0\}\bigr)\\
	&\cup\bigl(
	\{(\bx,y)\in\bbR^{d+1}\mid q(\bx,y)=0\}
	\cap
	\{(\bx,y)\in\bbR^{d+1}\mid r(\bx,y)=0\}\bigr)\\
	&\cup\bigl(
	\{(\bx,y)\in\bbR^{d+1}\mid q(\bx,y)<0\}
	\cap
	\{(\bx,y)\in\bbR^{d+1}\mid r(\bx,y)=0\}\bigr)
	\end{split}
\end{align}
is semialgebraic, as $g_p, q, r$ are all polynomial.
This shows that the Epanechnikov kernel, the biweight kernel,
and the triweight kernel are all semialgebraic,
and hence subanalytic. 
Since the graph of the tricube kernel 
$K(\cdot-\bx_i)=C_{\rm tc}\{(1-\|\cdot-\bx_i\|^3)_+\}^3$
with a positive normalizing constant $C_{\rm tc}$
can be written as 
\begin{align}
	\begin{split}
	&\{(\bx,y)\in\bbR^{d+1}\mid y=C_{\rm tc}\{(1-\|\bx-\bx_i\|^3)_+\}^3\}\\
	&=\bigl(
	\{(\bx,y)\in\bbR^{d+1}\mid q(\bx,y)>0\}
	\cap
	\{(\bx,y)\in\bbR^{d+1}\mid r(\bx,y)>0\}\\
	&\hphantom{=\bigl(\{}\cap
	\{(\bx,y)\in\bbR^{d+1}\mid g_{{\rm tc},1}(\bx,y)\coloneq 9\|\bx-\bx_i\|^6+6\|\bx-\bx_i\|^{12}+\|\bx-\bx_i\|^{18}-\{1+3\|\bx-\bx_i\|^6-(y/C_{\rm tc})\}^2=0\}\\
	&\hphantom{=\bigl(\{}\cap
	\{(\bx,y)\in\bbR^{d+1}\mid g_{{\rm tc},2}(\bx,y)\coloneq y-C_{\rm tc}=0\}\bigr)\\
	&\cup\bigl(
	\{(\bx,y)\in\bbR^{d+1}\mid q(\bx,y)>0\}
	\cap
	\{(\bx,y)\in\bbR^{d+1}\mid r(\bx,y)>0\}\\
	&\hphantom{=\bigl(\{}\cap
	\{(\bx,y)\in\bbR^{d+1}\mid g_{{\rm tc},1}(\bx,y)=0\}
	\cap
	\{(\bx,y)\in\bbR^{d+1}\mid g_{{\rm tc},2}(\bx,y)<0\}\bigr)\\
	&\cup\bigl(
	\{(\bx,y)\in\bbR^{d+1}\mid q(\bx,y)=0\}
	\cap
	\{(\bx,y)\in\bbR^{d+1}\mid r(\bx,y)=0\}\bigr)\\
	&\cup\bigl(
	\{(\bx,y)\in\bbR^{d+1}\mid q(\bx,y)<0\}
	\cap
	\{(\bx,y)\in\bbR^{d+1}\mid r(\bx,y)=0\}\bigr),
	\end{split}
\end{align}
where $g_{{\rm tc},1},g_{{\rm tc},2},q,r$ are polynomial functions,
this kernel is also semialgebraic
and hence subanalytic.

Although the kernel $K(\cdot-\bx_i)=C_{\rm -}\{(1-\|\cdot-\bx_i\|^2)_+\}^{3/2}$
with a positive normalizing constant $C_{\rm -}$
is not a piecewise polynomial kernel, one can show that this kernel is also semialgebraic 
and hence subanalytic, because its graph can be represented as
\begin{align}
	\begin{split}
	&\{(\bx,y)\in\bbR^{d+1}\mid y=C_{\rm -}\{(1-\|\bx-\bx_i\|^2)_+\}^{3/2}\}\\
	&=\bigl(
	\{(\bx,y)\in\bbR^{d+1}\mid q(\bx,y)>0\}
	\cap
	\{(\bx,y)\in\bbR^{d+1}\mid r(\bx,y)>0\}\\
	&\hphantom{=\bigl(\{}\cap
	\{(\bx,y)\in\bbR^{d+1}\mid g_{{\rm -},1}(\bx,y)\coloneq (1-\|\bx-\bx_i\|^2)^3-(y/C_{\rm -})^2=0\}\\
	&\hphantom{=\bigl(\{}
	\cap	
	\{(\bx,y)\in\bbR^{d+1}\mid g_{{\rm -},2}(\bx,y)\coloneq y-C_{\rm -}=0\}\bigr)\\
	&\cup\bigl(
	\{(\bx,y)\in\bbR^{d+1}\mid q(\bx,y)>0\}
	\cap
	\{(\bx,y)\in\bbR^{d+1}\mid r(\bx,y)>0\}\\
	&\hphantom{=\bigl(\{}\cap
	\{(\bx,y)\in\bbR^{d+1}\mid g_{{\rm -},1}(\bx,y)=0\}
	\cap
	\{(\bx,y)\in\bbR^{d+1}\mid g_{{\rm -},2}(\bx,y)<0\}\bigr)\\
	&\cup\bigl(
	\{(\bx,y)\in\bbR^{d+1}\mid q(\bx,y)=0\}
	\cap
	\{(\bx,y)\in\bbR^{d+1}\mid r(\bx,y)=0\}\bigr)\\
	&\cup\bigl(
	\{(\bx,y)\in\bbR^{d+1}\mid q(\bx,y)<0\}
	\cap
	\{(\bx,y)\in\bbR^{d+1}\mid r(\bx,y)=0\}\bigr)
	\end{split}
\end{align}
with the polynomial functions $g_{{\rm -},1},g_{{\rm -},2}, q, r$.

\textbf{Cosine Kernel}\\
The cosine kernel $K(\cdot-\bx_i)=C_{\rm cs}
\cos(\tfrac{\pi\|\cdot-\bx_i\|}{2})\bbI(\|\cdot-\bx_i\|\le1)$
with a positive normalizing constant $C_{\rm cs}$
is semianalytic and hence subanalytic, 
because its graph can be written as
\begin{align}
	\begin{split}
	&\{(\bx,y)\in\bbR^{d+1}\mid 
	y=C_{\rm cs}\cos(\tfrac{\pi\|\bx-\bx_i\|}{2})\bbI(\|\bx-\bx_i\|\le1)\}\\
	&=\bigl(
	\{(\bx,y)\in\bbR^{d+1}\mid q(\bx,y)>0\}
	\cap
	\{(\bx,y)\in\bbR^{d+1}\mid g_{{\rm cs}}(\bx,y)\coloneq C_{\rm cs}\cos(\tfrac{\pi\|\bx-\bx_i\|}{2})-y=0\}\bigr)\\
	&\cup\bigl(
	\{(\bx,y)\in\bbR^{d+1}\mid q(\bx,y)=0\}
	\cap
	\{(\bx,y)\in\bbR^{d+1}\mid r(\bx,y)=0\}\bigr)\\
	&\cup\bigl(
	\{(\bx,y)\in\bbR^{d+1}\mid q(\bx,y)<0\}
	\cap
	\{(\bx,y)\in\bbR^{d+1}\mid r(\bx,y)=0\}\bigr),
	\end{split}
\end{align}
where $g_{{\rm cs}}, q, r$ are analytic functions.

\end{document}